\documentclass[final]{siamltex}
\usepackage{todonotes}

\usepackage{verbatim}
\usepackage{latexsym, graphicx, epsfig}
\usepackage{pgfplots}
\usepackage[notcite,notref]{showkeys}
\usepackage{showkeys}
\usepackage{url}
\usepackage{color}
\usepackage{array}
\usepackage{amsfonts,amsmath,amssymb,bm}
\usepackage{verbatim}
\usepackage{setspace}
\usepackage{multirow}
\usepackage{slashbox}
\usepackage{algorithm}
\usepackage{algpseudocode}
\usepackage[T1]{fontenc}
\usepackage{appendix} 

\newcommand{\beq}{\begin{equation}}
\newcommand{\eeq}{\end{equation}}
\newcommand{\beqq}{\begin{equation*}}
\newcommand{\eeqq}{\end{equation*}}
\newcommand{\beqas}{\begin{eqnarray*}}
\newcommand{\eeqas}{\end{eqnarray*}}
\newcommand{\bsp}{\begin{split}}
\newcommand{\eesp}{\end{split}}

\newcommand{\bfa}[1]{\boldsymbol{#1}}

\usepackage{amscd}
\usepackage{colordvi}
\usepackage{epstopdf}
\usepackage{hyperref}
\usepackage{subcaption}
\usepackage{cleveref}
\def\norm#1{\|#1\|}

\title{WGFINN{s}: Weak formulation-based GENERIC formalism informed neural networks}

\author{
Jun Sur Richard Park\thanks{Division of Big Data Science, Korea University, Sejong, 30019, Republic of Korea (\texttt{junsurpark@korea.ac.kr}).}
\and
Auroni Huque Hashim\thanks{Department of Mathematics, North Carolina State University, Raleigh, NC 27695, USA (\texttt{ahashim3@ncsu.edu}, \texttt{yeonjong\_shin@ncsu.edu}).}
\and
Siu Wun Cheung\thanks{Center for Applied Scientific Computing, Lawrence Livermore National Laboratory, Livermore, CA 94550, USA. (\texttt{cheung26@llnl.gov},
\texttt{choi15@llnl.gov}).}
\and
Youngsoo Choi\footnotemark[3]
\and
Yeonjong Shin\footnotemark[2]
}

\begin{document}

\maketitle 

\begin{abstract}
Data-driven discovery of governing equations from noisy observations remains a fundamental challenge in scientific machine learning. 
While GENERIC formalism informed neural networks (GFINNs) provide a principled framework that enforces the laws of thermodynamics by construction, their reliance on strong-form loss formulations makes them highly sensitive to measurement noise.
To address this limitation, we propose \emph{weak formulation-based GENERIC formalism informed neural networks} (WGFINNs), which integrate the weak formulation of dynamical systems with the structure-preserving architecture of GFINNs. WGFINNs significantly enhance robustness to noisy data while retaining exact satisfaction of GENERIC degeneracy and symmetry conditions.
We further incorporate a state-wise weighted loss and a residual-based attention mechanism to mitigate scale imbalance across state variables.
Theoretical analysis contrasts quantitative differences between the strong-form and the weak-form estimators. Mainly, the strong-form estimator diverges as the time step decreases in the presence of noise, while the weak-form estimator can be accurate even with noisy data if test functions satisfy certain conditions.
Numerical experiments demonstrate that WGFINNs consistently outperform GFINNs at varying noise levels, achieving more accurate predictions and reliable recovery of physical quantities.
\end{abstract}

\begin{keywords}
data-driven discovery, GENERIC formalism, interpretable scientific machine learning, weak formulation
\end{keywords}

\pagestyle{myheadings}
\thispagestyle{plain}




\section{Introduction}

The discovery of governing equations for dynamical systems from observed data is a cornerstone of modern predictive modeling. 
Over the past two decades, data-driven modeling has increasingly evolved by enabling the approximation of complex, high-dimensional nonlinearities to harness available data for scientific computing to complement or advance data-free approaches. 
Many novel methods have been proposed using classical data-driven and machine learning approaches, 
such as dynamic mode decomposition (DMD)
\cite{schmid2010dynamic,rowley2009spectral,tu2014dynamic,proctor2016dynamic}, 
operator inference (OpInf)
\cite{peherstorfer2016data,mcquarrie2021data,mcquarrie2023nonintrusive}, 
symbolic regression 
\cite{bongard2007automated,schmidt2009distilling}, 
and sparse regression
\cite{brunton2016discovering, kang2019ident}. 
Building on these methods, Bayesian extensions
\cite{raissi2018hidden, zhang2018robust, hirsh2022sparsifying}
and neural network (NN) approaches 
\cite{raissi2018hidden, wu2020data, bonneville2022bayesian, stephany2022pde, chen2021physics}
have become increasingly popular. 
These contemporary methods have been combined with dimensionality reduction \cite{lusch2018deep, champion2019data} 
and applied to flow dynamics modeling 
\cite{kim2019deep, zhang2019deep, wang2020reduced, cheung2020deep, wang2020efficient}.
The combination has been extended to the Latent Space Dynamics Identification (LaSDI) framework 
\cite{fries2022lasdi,he2023glasdi,bonneville2024gplasdi,park2024tlasdi, bonneville2024comprehensive}. 

While the aforementioned data-driven modeling strategies can achieve impressive accuracy on specific training datasets, 
they are pure data-driven methods, which learn governing equations and mimic observed trajectories without any prior knowledge of the underlying physics, and are therefore notorious for failing to generalize. 
When faced with regions of the state space where data is scarce, or when asked to project into the future, purely statistical models often produce physically impossible results, such as violating energy conservation or predicting negative entropy production. 
As an alternative, physics-informed data-driven modeling attempts to bridge this gap by embedding fundamental principles into the learning process. By leveraging physical laws, these models require significantly less data to reach high performance, maintain better stability, and exhibit superior generalization.
One of the most robust frameworks for this is the General Equation for Non-Equilibrium Reversible-Irreversible Coupling (GENERIC) formalism \cite{grmela1997dynamics,ottinger1997dynamics,ottinger2005beyond}. The formalism describes the evolution of a system beyond equilibrium by separating dynamics into reversible and irreversible components. These components are governed by specific generators that must satisfy symmetry and degeneracy conditions, effectively encoding the first and second laws of thermodynamics into the very structure of the differential equations, and can be imposed either softly or through hard-coded architectures. 
Soft constraints \cite{hernandez2021structure, hernandez2021deep, lee2021machine} utilize regularization terms in a loss function to penalize deviations from physical laws, but they do not guarantee that the resulting model will strictly obey these laws after training. 
In contrast, in the work of GENERIC-informed neural networks (GFINNs) \cite{zhang2022gfinns}, hard constraints involve designing neural network architectures that satisfy physical principles by construction and ensure that the learned system remains thermodynamically consistent. 

Yet, such advanced neural architectures still face two challenges: the quality of the data and the method of optimization. Most existing approaches rely on the strong form of the governing equations, meaning that the optimization process attempts to fit the pointwise time-derivatives of the system. In practical applications, these derivatives are rarely available directly and must be approximated from noisy measurement data using numerical differentiation. It is well-documented in numerical analysis that differentiation is an ill-posed operation that amplifies even minor noise, often leading to a total collapse of the identification process when the signal-to-noise ratio is low. 
To address this numerical instability, a separate lineage of research has focused on the weak form of system identification \cite{gurevich2019robust, pantazis2019unified, wang2019variational, messenger2021weak, messenger2021weak-jcp, bortz2023direct, tang2023weakident, tran2024weak, lopez2025weak, tang2025wg}. Instead of fitting derivatives, the weak formulation utilizes a Galerkin-based approach where the dynamics are integrated against compactly supported test functions. 
Through the integration by parts, the derivative is shifted from the noisy data onto a smooth, analytically defined test function, which are found to be effective in identifying dynamics under noise.
However, such approaches \cite{messenger2024weak,bortz2024weak} require a predefined library of carefully selected candidate functions and represent dynamics as a linear combination of them, which is the so-called dictionary method. If the true physical mechanism of a system is not represented in the library, the dictionary method fails.

The present work considers a non-dictionary approach in learning dynamical systems under noise.
In particular, we propose the Weak formulation-based GFINNs, namely, WGFINNs that combines the thermodynamic structures of the GENERIC formalism with the noise robustness of the weak formulation. While learning dynamics by means of neural networks (NNs) offers flexibility and expressive power in modeling, the presence of noise creates a new layer of challenges as it hinders learning processes towards correct dynamics. To overcome such challenges, we develop a weak formulation-based learning framework, which in principle can be applicable for any NN models, while we focus on using GFINNs.
The learning framework uses a state-wise weighted loss function and applies the state-wise residual-based attention (State-RBA) technique \cite{anagnostopoulos2024residual} to overcome the inherent scale imbalance across different state variables.
On a theoretical side, we provide a mathematical justification on the use of weak formulation in Proposition~\ref{prop:weak-form-motivation}, which reveals the trajectory-wise consistency. That is, the weak formulation loss yields a dynamical system that produces the same trajectory to the data.
By focusing on the univariate linear dynamics (parameter estimation), we analyze the errors of both the strong and the weak formulation-based losses.
For the strong form loss, we prove that the error diverges as the time-step size decreases to zero while there is a special time-step size such that the error is zero (perfect recovery).
For the weak form loss, we prove that there exists a test function which yields the weak form solution diverges as the support of the test function shrinks to zero, while converges to the underlying true parameter as the support gets larger.
While both forms yield divergence under similar conditions, their conditions for convergence differ significantly. See Section~\ref{sec:analysis} for details, especially Figure~\ref{fig:weak-form-illustration}.
Extensive numerical experiments demonstrate the robustness of the proposed method where the data is corrupted by different levels of noise. 
The learned dynamical system is not only thermodynamically consistent by design, satisfying energy conservation and entropy laws, but is also capable of being learned directly from noisy data.

The organization of the paper is as follows: In Section \ref{sec:prelim}, we provide the problem setup and preliminaries, introducing both the strong and weak formulation-based loss functions. Section \ref{sec:method} presents our proposed WGFINNs, detailing the learning framework along with the residual-based attention mechanism. Section \ref{sec:analysis} provides a mathematical analysis of parameter estimation that compares the quantitative differences between the strong and the weak form estimators. In Section \ref{sec:results}, we demonstrate the effectiveness of WGFINNs through several numerical experiments. Finally, Section \ref{sec:conclusion} concludes the paper with a summary and discussion of future work.
\section{Problem setup and preliminary}
\label{sec:prelim}
Let us consider the problem of learning a dynamical system of the following form:
\begin{equation}
\label{eq:strong}
\begin{split}
\dot{\bfa{x}}(t) &= \bfa{f}(\bfa{x}(t)), \quad \bfa{x} \in \Omega \in \mathbb{R}^d, \ t \in (0,T], \ 
\bfa{x}(0) = \bfa{x}_0,
\end{split}
\end{equation}
where $\bfa{f}: \mathbb{R}^d \to \mathbb{R}^d$ is the unknown vector-valued function. The goal is to construct an approximation to $\bfa{f}$ by utilizing a set of trajectory data.

For an initial state $\bfa{x}_0$, the solution to \eqref{eq:strong} is denoted by $\bfa{x}(t;\bfa{x}_0)$.
Given a set of initial states, $\{\bfa{x}_0^{(i)}\}_{i=1}^{N_\text{traj}}$, let $\bfa{x}_k^{(i)} = \bfa{x}(t_k;\bfa{x}_0^{(i)})$ be the solution at time $t_k$, where $t_k = k \Delta t$ with the timestep $\Delta t$ such that $T = K\Delta t$.
Then, the collection of solution trajectories is given by 
$\{ \{\bfa{x}_k^{(i)} \}_{k=0}^K \}_{i=1}^{N_\text{traj}}$.
In practice, however, one can only access to noisy measurements or trajectories; each observed state is 
\begin{equation} \label{def:noisy-data}
    \bfa{y}_k^{(i)} = \bfa{x}_k^{(i)} + \bfa{\epsilon}^{(i)}_k,
\end{equation}
where $\bfa{\epsilon}^{(i)}_k$'s are independent random noise vectors.

Our objective is to approximate $\bfa{f}$ from the noisy trajectory data $\{ \{\bfa{y}_k^{(i)} \}_{k=0}^K \}_{i=1}^{N_\text{traj}}$ by means of an NN model $\bfa{f}_{\text{NN}}(\cdot,\bfa{\theta})$.
Specifically, we are interested in a NN model that preserves the laws of thermodynamics.
See Section~\ref{sec:GFINNs}.

Let us consider the dynamical system defined through $\bfa{f}_{\text{NN}}(\cdot,\bfa{\theta})$:
\begin{equation}
\label{eq:strong-NN}
\begin{split}
\dot{\widehat{\bfa{x}}}(t) &= \bfa{f}_\text{NN}(\widehat{\bfa{x}}(t);\bfa{\theta}), \quad \widehat{\bfa{x}} \in \Omega \in \mathbb{R}^d, \ t \in (0,T], \ 
\widehat{\bfa{x}}(0) = \bfa{x}_0.
\end{split}
\end{equation}
The solution to \eqref{eq:strong-NN} is denoted by $\widehat{\bfa{x}}(t;\bfa{x}_0,\bfa{\theta})$ whose initial state is $\bfa{x}_0$ and the dependency of the NN parameter $\bfa{\theta}$ is explicitly written.
We consider the one-step evolution from $\bfa{y}_{k-1}^{(i)}$ governed by \eqref{eq:strong-NN} which will be used in defining loss functions, and denote it by 
\begin{equation} \label{def:one-step-NNprediction}
    \widehat{\bfa{x}}_k^{(i)}(\bfa{\theta};\bfa{y}^{(i)}_{k-1}) := \widehat{\bfa{x}}(\Delta t;\bfa{y}_{k-1}^{(i)},\bfa{\theta}).
\end{equation}
Note that if $\bfa{\theta}$ is chosen to satisfy $\bfa{f}_\text{NN}(\cdot;\bfa{\theta}) = \bfa{f}$ and $\bfa{y}_{k-1}^{(i)} = \bfa{x}_{k-1}^{(i)}$ (i.e, noise-free), $\widehat{\bfa{x}}_k^{(i)}(\bfa{\theta})$ and $\bfa{x}_k^{(i)}$ should be identical.

\subsection{Strong formulation-based loss function}\label{sec:strong_form}
A popular approach to train the NN model $\bfa{f}_\text{NN}$ is to minimize the mean squared error (MSE) loss function based on the strong formulation of the dynamical system \cite{zhang2022gfinns,hernandez2021structure,lee2021machine}:
\begin{equation} \label{def:strong-form-loss}
\mathcal{L}_W^\text{strong}(\bfa{\theta}) = \frac{1}{N_{\text{traj}}} \sum\limits_{i=1}^{N_\text{traj}} \frac{1}{K} \sum\limits_{k=0}^{K} \frac{1}{d}\norm{\widehat{\bfa{x}}^{(i)}_k(\bfa{\theta};\bfa{y}^{(i)}_{k-1}) - \bfa{y}^{(i)}_k}_{W}^2,
\end{equation}
where $\bfa{y}^{(i)}_k$'s are noisy observation of states \eqref{def:noisy-data} and $\widehat{\bfa{x}}^{(i)}_k(\bfa{\theta};\bfa{y}^{(i)}_{k-1})$'s are corresponding NN predictions \eqref{def:one-step-NNprediction}.
Here, $\|\cdot\|_W$ is the $W$-norm defined by $\|x\|_W = \sqrt{x^\top Wx}$ where $W = \text{diag}(w_1^2, w_2^2, \ldots, w_d^2)$ whose diagonal entries are strictly positive.



While this strong formulation-based loss function \eqref{def:strong-form-loss} is popularly used and found to be effective when it comes to noiseless data, it fails miserably otherwise as illustrated, for example, in Figure~\ref{fig:LD_WGFINNs_vs_GFINNs_noise_example}. 

\subsection{Weak formulation-based loss function}
\label{sec:weak_form_setup}
The weak formulation-based approaches have been popularly used when it comes to handling noisy data.
In particular, \cite{tran2024weak,messenger2021weak} proposed a loss function based on the weak formulation, which was found to be robust against noise for learning dynamical systems. 
We note that the main difference between the current work and \cite{tran2024weak,messenger2021weak} is the modeling of the dynamical system. 
While the current work focuses on using NN models in constructing the dynamics, \cite{tran2024weak,messenger2021weak} follow the so-called dictionary approach which expresses the dynamical system as a linear combination of selective elements from a pre-defined dictionary or set.

Let $V$ be the space of smooth real-valued functions with compact support on $(0,T)$. Any element $\phi$ in $V$ is called a test function. 
By multiplying $\phi$ and integrating over the entire temporal domain on both sides of (\ref{eq:strong}), it follows from the integration by parts that 
\begin{equation}
\label{eq:weak}
-\int_0^T  \dot{\phi}(t) \bfa{x}(t) dt = \int_0^T \bfa{f}(\bfa{x}(t)) \phi(t) dt, \quad \forall \phi \in V,
\end{equation}
which is the weak formulation of \eqref{eq:strong}.

\begin{proposition} \label{prop:weak-form-motivation}
    Suppose that $\bfa{f}$ and $\bfa{f}_\text{NN}(\cdot;\bfa{\theta})$ are continuous functions.
    Let $\bfa{x}(t;\bfa{x}_0)$ and $\widehat{\bfa{x}}(t;\bfa{x}_0,\bfa{\theta})$ be the solution to 
    \eqref{eq:strong} and \eqref{eq:strong-NN}, respectively, with the same initial state $\bfa{x}_0$.
    Then, the following statement holds. 
    The two trajectories are identical, i.e., $\bfa{x}(t;\bfa{x}_0)=\widehat{\bfa{x}}(t;\bfa{x}_0,\bfa{\theta})$ for all $t \in [0,T]$
    if and only if 
    $\bfa{f}(\cdot)$ on the trajectory $\bfa{x}(t;\bfa{x}_0)$ matches with $\bfa{f}_\emph{\text{NN}}(\cdot)$ on $\widehat{\bfa{x}}(t;\bfa{x}_0,\bfa{\theta})$.
    Moreover, if the above is the case, we have
    \begin{equation} \label{eq:weak-form-loss-motivation}
        -\int_0^T  \dot{\phi}(t) \bfa{x}(t;\bfa{x}_0) dt = \int_0^T \bfa{f}_\emph{\text{NN}}(\bfa{x}(t;\bfa{x}_0);\bfa{\theta}) \phi(t) dt, \quad \forall \phi \in V.
    \end{equation}
\end{proposition}
\begin{proof}
    For ease of discussion, we omit the dependency of $\bfa{\theta}$ throughout the proof.
    Suppose that $\bfa{x}(t;\bfa{x}_0)=\widehat{\bfa{x}}(t;\bfa{x}_0)$ for all $t \in [0,T]$.
    It is then readily checked that 
    $\bfa{f}(\bfa{x}(t)) = \dot{\bfa{x}}(t) = \dot{\widehat{\bfa{x}}}(t) = \bfa{f}_\text{NN}(\widehat{\bfa{x}}(t))$,
    which proves one direction.

    Let us assume that $\bfa{f}(\bfa{x}(t;\bfa{x}_0)) = \bfa{f}_{\text{NN}}(\widehat{\bfa{x}}(t;\bfa{x}_0))$.
    It then follows from \eqref{eq:weak} that 
    \begin{align*}
        -\int_0^T  \dot{\phi}(t) \bfa{x}(t) dt = \int_0^T \bfa{f}(\bfa{x}(t)) \phi(t) dt
        = \int_0^T \bfa{f}_\text{NN}(\widehat{\bfa{x}}(t)) \phi(t) dt = -\int_0^T  \dot{\phi}(t) \widehat{\bfa{x}}(t) dt,
    \end{align*}
    for all $\phi \in V$.
    Since $\dot{V}:= \{\dot{\phi} | \phi \in V\} = V$ (the fundamental theorem of calculus), it follows from the fundamental lemma of calculus of variations that 
    $\bfa{x}(t) = \widehat{\bfa{x}}(t)$ for all $t \in [0,T]$.
    Hence, \eqref{eq:weak-form-loss-motivation} readily follows.
\end{proof}

Proposition~\ref{prop:weak-form-motivation} reveals an important equation \eqref{eq:weak-form-loss-motivation} that connects the data-driven dynamics \eqref{eq:strong-NN} and the target dynamics \eqref{eq:strong} through the weak formulation. It serves as a master equation for the weak formulation-based loss function \cite{tran2024weak,messenger2021weak}.
In addition, the result implies that one can expect to learn the unknown dynamics at least along trajectories. 
This offers a mathematical justification of the weak formulation in learning a dynamical system, which is absent in the literature to the best of our knowledge.

To make \eqref{eq:weak-form-loss-motivation} practical, one needs to approximate (1) the integrals and (2) the test function space $V$.
Following the standard approaches, we employ a quadrature rule to numerically compute the integral terms and construct a finite dimensional subspace $V_J \subset V$ spanned by $\{\phi_j\}_{j=1}^J$.

For a quadrature rule $\{s_q, w_q\}_{q=0}^{Q+1}$, the integral is approximated by
$\int_0^T \dot{\phi}(t)\bfa{x}(t)dt \approx \sum_{q} w_q \dot{\phi}(s_q)\bfa{x}(s_q)$.
This, however, requires the state values $\{\bfa{x}(s_q)\}_{q=0}^{Q+1}$ at every quadrature point, which may not be available in practice. 
Thus, the trapezoidal rule is popularly used as it requires a minimal point evaluation. 
Consider the following trapezoidal rule: $s_q = t^* + hq$ where $t^*=n_0\Delta t$ and $h=\bar{n}\Delta t$ for some integers $n_0, \bar{n}$ such that $s_{Q+1} \le T$,
and $w_q = h$ if $q \ne 0, Q+1$ and $\frac{1}{2}h$ if $q = 0, Q+1$.
Then, \eqref{eq:weak-form-loss-motivation} becomes 
\begin{align*}
    -\sum_{q=1}^Q h \dot{\phi}(s_q)\bfa{x}(s_q) = \sum_{q=1}^Q h\phi(s_q)\bfa{f}_{\text{NN}}(\bfa{x}(s_q);\bfa{\theta}),
\end{align*}
where the test function and its derivative vanish at $s_0$ and $s_{Q+1}$.
The weak form loss is then defined as a MSE loss over test functions and noisy trajectory data. That is, the NN model is trained to minimize the following loss:
\begin{equation}  \label{def:weak-form-loss}
\mathcal{L}_W^\text{weak}(\bfa{\theta}) = \frac{1}{J\cdot d\cdot N_{\text{traj}}} \sum\limits_{i=1}^{N_\text{traj}} \left\Vert \bfa{Y}^{(i)}\dot{\bfa{\Phi}} + \bfa{F}_{\text{NN}}^{(i)}(\bfa{\theta}) \bfa{\Phi} \right\Vert_{W}^2,
\end{equation}
where $\|M\|_{W}^2 := \sum_{j} \|M_{:,j}\|_W^2$ is a weighted Frobenius norm where $M_{:,j}$ is the $j$-th column of $M$, 
$\bfa{\Phi}, \, \dot{\bfa{\Phi}}$ are matrices whose $(q,j)$-components are defined by $\bfa{\Phi}_{q j} = h\phi_j(s_q)$ 
and $\dot{\bfa{\Phi}}_{q j} = h\dot{\phi}_j(s_q)$, respectively, and 
\begin{equation*}
\bfa{Y}^{(i)} = \begin{bmatrix}
    \bfa{y}_{\bar{n}}^{(i)} & \cdots & \bfa{y}_{\bar{n}Q}^{(i)} 
\end{bmatrix}, \quad 
\bfa{F}_{\text{NN}}^{(i)}(\bfa{\theta}) = 
\begin{bmatrix}
    \bfa{f}_{\text{NN}}(\bfa{y}_{\bar{n}}^{(i)};\bfa{\theta}) & \cdots & \bfa{f}_{\text{NN}}(\bfa{y}_{\bar{n}Q}^{(i)};\bfa{\theta})
\end{bmatrix}.
\end{equation*}

\subsubsection{Test functions}\label{sec:test_func}




The choice of test functions is crucial for the effectiveness of the weak-form approaches. Common choices in the literature include orthogonal test functions generated from a set of multiscale $C^{\infty}$ bump functions \cite{bortz2023direct,heitzman2025practical} and Hartley modulating functions \cite{heitzman2025practical,boulier2014algorithm}. See also \cite{tran2025weak}. In this work, we employ unimodal piecewise polynomial test functions used in \cite{messenger2021weak,tran2024weak,tang2023weakident,he2023glasdi,heitzman2025practical} due to their robustness to noise \cite{messenger2021weak,heitzman2025practical}, computational efficiency and flexibility in controlling both the support width and the degree of smoothness at the support endpoints.

Specifically, each test function is defined as
\begin{equation*}
\label{eq:test_func}
\phi(t) = \begin{cases}
    C(t-a)^p(b-t)^p &\text{if } t \in [a,b], \\ 
    0 &\text{otherwise,}
\end{cases}
\end{equation*}
where $a<b, \, (a,b)\subset [0,T]$. 
$p \geq 0$ is a non-negative integer determining the polynomial degree and boundary smoothness, which we refer to as the polynomial degree parameter.
In this work, we set $p \geq 2$ in order to ensure that the test function and its derivative vanish at the boundary points $a$, $b$.
$C$ is a normalization constant satisfying $\norm{\phi}_\infty = 1$. 
As $p$ gets larger, the decay is faster towards the endpoints of the support.

The test functions are characterized by three hyperparameters, $\ell, p, s$: the support width $\ell$, the polynomial degree $p$, and the overlap parameter $s \in [0,1]$. The support width $\ell$ denotes the number of time points in each test 
function's support. Once $\ell$ is determined, the overlap parameter $s$ controls the degree of overlap between consecutive test functions. Specifically, neighboring test functions share $\ell_{\text{overlap}} = \lfloor \ell \cdot (1 - 
\sqrt{1 - s^{1/p}})\rfloor$
time points in their supports. Larger values of $s$ result in greater overlap: $s = 0$ yields disjoint supports, while $s \to 1$ approaches 
complete overlap between neighboring functions. For detailed guidance on selecting these hyperparameters, we refer to \cite{messenger2021weak}. The number of test functions $J$ is then determined by 
\(
J = 1 + \left\lfloor (K+1 - \ell)/(\ell - \ell_{\text{overlap}}) \right\rfloor.
\)
In this work, the hyperparameters $\ell$, $p$, and $s$ are individually tuned for each experiment to optimize performance. The specific choices used in our numerical experiments are reported in Appendix~\ref{appendix:training_details}.

\subsection{GENERIC formalism informed neural networks}
\label{sec:GFINNs}
When no physical principles are available for a given system, a NN model $\bfa{f}_\text{NN}$ is set to one of the popular neural network architectures, such as feed-forward neural networks (FNNs), residual neural networks (ResNets) \cite{he2016deep}, recurrent neural networks (RNNs) \cite{pearlmutter1989learning,rumelhart1985learning}, Kolmogorov-Arnold networks (KANs) \cite{liu2024kan} to name just a few.
This approach is generally referred to as purely data-driven as no explicit physical or meaningful structures are imposed. 
Another approach is to incorporate any available prior knowledge into the NN model so that the model itself preserves the underlying structures while it is learned from data.
In this work, we focus on the latter case of the physics-informed data-driven approach where we focus on the laws of thermodynamics. In particular, we employ the GENERIC formalism informed neural networks (GFINNs) \cite{zhang2022gfinns,park2024tlasdi} for the NN modeling choice of $\bfa{f}_\text{NN}$ and briefly introduce it in this subsection.

The GENERIC formalism \cite{grmela1997dynamics,ottinger2005beyond} is a mathematical framework that integrates both conservative and dissipative systems, providing a comprehensive description of beyond-equilibrium thermodynamics. It defines the dynamics through the scalar-valued functions of the total energy ($E$) and the entropy ($S$) of the system, and the matrix-valued functions of the Poisson ($L$) and friction ($M$) matrices:
\begin{equation}
\label{eq:GENERIC}
\begin{split}
&\dot{\bfa{x}} = L(\bfa{x}) \nabla E(\bfa{x}) +M(\bfa{x}) \nabla S(\bfa{x}), \\
\textrm{subject to} \ \ &L(\bfa{x}) \nabla S(\bfa{x}) = M(\bfa{x}) \nabla E(\bfa{x}) = \bfa{0}, \\
&L(\bfa{x}) \ \textrm{is skew-symmetric, i.e.,} \ L(\bfa{x}) = -L(\bfa{x})^\top,\\
&M(\bfa{x}) \ \textrm{is symmetric} \ \& \ \textrm{positive semi-definite.} 
\end{split}
\end{equation}
These four functions should satisfy the so-called degeneracy conditions, which ensure the first and second laws of thermodynamics, i.e., the energy conservation $\frac{d}{dt}E(\bfa{x}) = 0$ and non-decreasing entropy $\frac{d}{dt}S(\bfa{x}) \geq 0$.
The conditions can be viewed as the underlying physical structures of the dynamical systems.

The GFINNs \cite{zhang2022gfinns} are a class of NN models that follow the GENERIC formalism, which exactly satisfy the degeneracy conditions by construction while also being sufficiently expressive to learn the underlying dynamics from data.
While readers can consult \cite{zhang2022gfinns} for more details, we briefly review the GFINNs.
The GFINNs consist of four NNs -- $E_{\text{NN}}$, $S_{\text{NN}}$, $L_{\text{NN}}$ and $M_{\text{NN}}$ -- that approximate the functions $E$, $S$, $L$, and $M$ in (\ref{eq:GENERIC}), respectively.
The construction of the NNs is based on the following result \cite[Lemma 3.6]{zhang2022gfinns}, which shows a way to enforce the degeneracy conditions by means of skew-symmetric matrices.
\begin{theorem}\label{lemma:skewQ}
Let $A_i \in \mathbb{R}^{d\times d}$ be a skew-symmetric matrix, for $i = 1,\dots, M$. For a differentiable scalar function $h(\cdot): \mathbb{R}^{d} \to \mathbb{R}$, we 
consider the matrix valued function $Q_h(\cdot): \mathbb{R}^{d} \to \mathbb{R}^{M \times d}$ whose $i$-th row is defined as $(A_i \nabla h(\cdot))^\top$. Then we have $Q_h(\bfa{x}) \nabla h(\bfa{x}) = \bfa{0}$ for all $\bfa{x} \in \mathbb{R}^{d}$. 
\end{theorem}

For a scalar-valued NN, $S_{\textnormal{NN}}$, which represents the entropy,
let $Q_{S_{\textnormal{NN}}}$ be the corresponding matrix-valued function, defined in Theorem \ref{lemma:skewQ}.
Then $L_{\textnormal{NN}}$ for the Poisson matrix is constructed by 
\begin{equation*}
\label{eq:ann_structure}
L_{\textnormal{NN}}(\bfa{x}) := (Q_{S_{\textnormal{NN}}}(\bfa{x}))^\top \, B_{\textnormal{NN}}(\bfa{x}) \, Q_{S_{\textnormal{NN}}}(\bfa{x}),
\end{equation*}
where $B_{\textnormal{NN}}(\bfa{x}): \mathbb{R}^{d} \to  \mathbb{R}^{M \times M}$ is skew symmetric.
The NN models for $E_{\text{NN}}$ and $M_{\text{NN}}$ are similarly constructed.
It then can be checked that the resulting GFINNs satisfy degeneracy conditions (\ref{eq:GENERIC}) exactly by construction. 



While GFINNs provide a principled and expressive framework for learning thermodynamically consistent dynamics by exactly enforcing the GENERIC structure, the training procedure has been based on the strong formulation of \eqref{def:strong-form-loss} with $W = I_d$ \cite{zhang2022gfinns,park2024tlasdi,he2026thermodynamically}.
As discussed earlier, such strong-form formulations are known to be sensitive to measurement noise.

\section{Weak formulation-based learning framework: WGFINNs}
\label{sec:method}

We propose a weak formulation-based learning framework that effectively trains a NN model from noisy trajectory data. 
While the framework applies to any NN models, we focus on GFINNs as it preserves the thermodynamic structures through the GENERIC formalism, which we refer to as the weak formulation-based GFINNs (WGFINNs).

The framework introduces the state-wise residual-based attention (State-RBA) \cite{anagnostopoulos2024residual} mechanism that facilitates scale-agnostic learning across all state variables. 
State-RBA is found to be effective, especially, when it comes to learning dynamical systems with different scales such as systems of thermodynamics.
For instance, mechanical quantities such as position and momentum may have significantly larger values compared to thermodynamic quantities such as entropy. Without proper weighting, standard loss functions tend to prioritize the reduction of errors in large-magnitude variables, potentially neglecting smaller but equally important variables during optimization. This issue is illustrated in Figure \ref{fig:LD_WGFINNs_vs_GFINNs_weight_example}, which compares the predictions obtained by WGFINNs trained with weighted and unweighted loss functions ($W = I$) on \textit{noise-free} data for a linearly damped system (Section \ref{sec:linearly_damped}). 
When using an unweighted loss, the entropy variable $S$ with small magnitude exhibits noticeable deviations from the ground truth, despite accurate predictions for the position $q$ and momentum $p$. In contrast, when state-wise weights are incorporated into the loss function, all three variables—$q$, $p$, and $S$—are learned with comparable accuracy for both methods. This demonstrates that appropriate state-wise weighting is essential to ensure balanced learning across all components of the state vector, particularly for variables that play a critical role in the thermodynamic consistency of the system.
\begin{figure}[!ht]
	\centering
 \includegraphics[width=\textwidth]{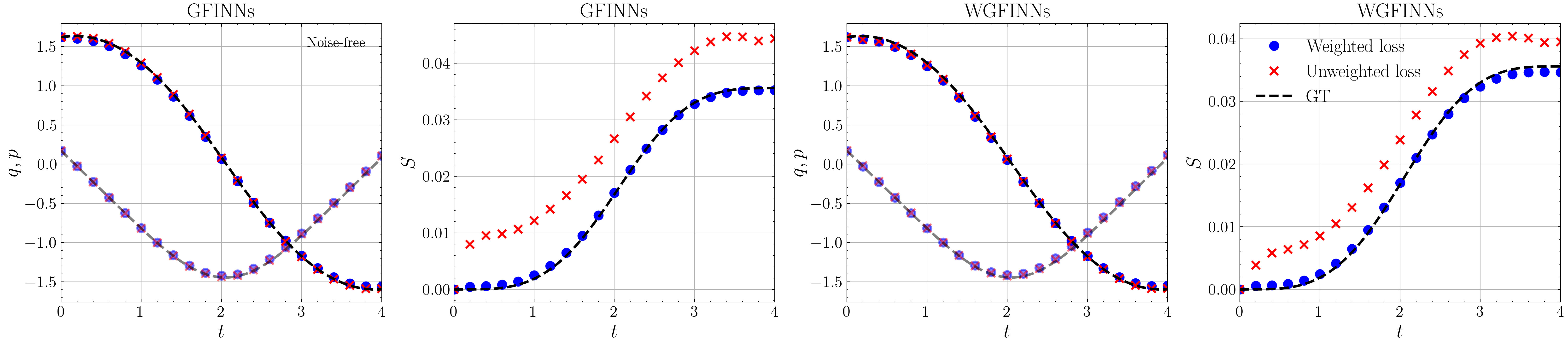}
  \caption{Example \ref{sec:linearly_damped}. Comparison of weighted and unweighted loss functions for noise-free data. Predictions by GFINNs and WGFINNs for all state variables.}
\label{fig:LD_WGFINNs_vs_GFINNs_weight_example}
\end{figure}

While State-RBA is applicable for both the strong and the weak form losses, it does not resolve issues of learning from noisy data.
It is the weak form loss that brings the noise-resilient feature in the learning process.
Figure~\ref{fig:LD_WGFINNs_vs_GFINNs_noise_example} illustrates the performance difference between the strong form loss (GFINNs) and the weak form loss (WGFINNs). 
It is clearly seen that GFINN prediction (red crosses) deviates substantially from the ground truth, while WGFINN (blue circles) yields accurate predictions throughout despite the presence of 5\% noise in the training data.
Note that both WGFINN and GFINN employ the identical NN architecture and the only difference comes from how they are trained.
Further the experiment details can be found in Section~\ref{sec:linearly_damped}. 
\begin{figure}[!ht]
	\centering
 \includegraphics[width=\textwidth]{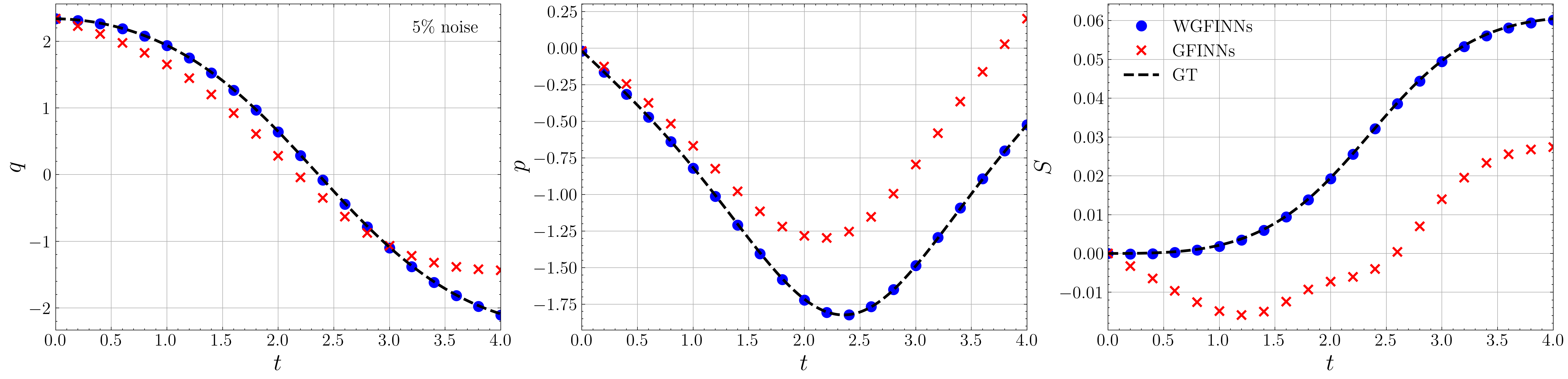}
  \caption{Example \ref{sec:linearly_damped}. Ground-truth (GT) trajectories and the corresponding predictions by WGFINNs and GFINNs with 5$\%$-noise corruption.}
\label{fig:LD_WGFINNs_vs_GFINNs_noise_example}
\end{figure}




\subsection{State-wise weighted loss and residual-based attention}
\label{sec:sw_rba}
Given a NN model of choice (this work uses GFINNs), the framework consists of two key components. One is to construct an appropriate weak-form loss function and the other is the use of State-RBA that dynamically adjusts the weights in the weak-form loss.
As the weak-form loss is already presented in Section~\ref{sec:weak_form_setup}, in what follows, we describe the State-RBA training scheme.



To ensure balanced learning across all state components, we employ the weighted norm $\|\cdot\|_W$ where the weight matrix is defined through the standard deviations of the training data. That is, $W = \text{diag}(\frac{1}{\sigma_1}, \frac{1}{\sigma_2}, \ldots, \frac{1}{\sigma_d})$, where $\sigma_\ell$ is the standard deviation of the $\ell$-th state variable across all trajectories and time steps, i.e.,
\begin{equation*}
\label{eq:std_l}
\sigma_\ell^2 = \frac{1}{N_{\text{traj}}(K+1)} \sum_{i=1}^{N_{\text{traj}}} \sum_{k=0}^{K} \left((\bfa{y}_k^{(i)})_\ell - \mu_\ell\right)^2, \
\mu_\ell = \frac{1}{N_{\text{traj}}(K+1)} \sum_{i=1}^{N_{\text{traj}}} \sum_{k=0}^{K} (\bfa{y}_k^{(i)})_\ell,
\end{equation*}
where $\ell = 1, \ldots, d$.

To further refine the state-wise weighting dynamically, we propose a state-wise RBA mechanism, originally proposed in \cite{anagnostopoulos2024residual}. 
State-RBA adaptively adjusts the weights during training based on the current residuals of each state variable, ensuring that all variables receive appropriate attention throughout the optimization process. 
Specifically, with the initial multiplier $\bfa{\lambda}^{0} = \bm{1}$ and an initial NN parameters $\bfa{\theta}^0$, the gradient descent algorithm updates them according to the following rules: For $k=1,\dots,$
\begin{align*}
\bfa{\lambda}^{k}
    &=
    \gamma \bfa{\lambda}^{k-1}
    +
    \eta^{\ast}
    \sqrt{W}\bfa{e}^{k-1}/\|\sqrt{W}\bfa{e}^{k-1}\|_{\infty}, \\
    \bfa{\theta}^{k} &= \bfa{\theta}^{k-1} - \eta \nabla_{\bfa{\theta}} \mathcal{L}^{\textrm{str}}_{W\textrm{diag}(\bfa{\lambda}^{k})}(\bfa{\theta}^{k-1}),
\end{align*}
where $\gamma \in (0,1]$ is a decay parameter, $\eta^{\ast}$ is the learning rate associated with the weighting scheme,
$\| \bfa{v} \|_{\infty}
=
\max\limits_{\ell = 1,\dots,d} |v_\ell|$ for a vector $\bfa{v} \in \mathbb{R}^d$, 
$\eta$ is the learning rate, and
$\text{str} \in \{\text{strong}, \text{weak}\}$. 
Here $\bfa{e}^k$ denotes the residual of the loss at $\bfa{\theta}^k$.
If the weak form loss is used, it is defined by
\begin{equation*}
\label{eq:RBA_ei}
\bfa{e}^k
:=
\frac{1}{N_{\mathrm{traj}}\,J}
\sum_{i=1}^{N_{\mathrm{traj}}}
\sum_{j=1}^{J}
\left|
\left(
\bfa{Y}^{(i)}\dot{\bfa{\Phi}} + \bfa{F}_{\text{NN}}^{(i)}(\bfa{\theta}^k) \bfa{\Phi} 
\right)_{:, j}
\right|,
\end{equation*}
and if the strong form loss is used, it is given by
\begin{equation*} \label{eq:RBA_ei_gfinns}
\bfa{e}^k
:=  \frac{1}{N_{\text{traj}}} \sum\limits_{i=1}^{N_\text{traj}} \frac{1}{K} \sum\limits_{k=1}^{K} \left| \widehat{\bfa{x}}^{(i)}_k(\bfa{\theta}^k;\bfa{y}^{(i)}_{k-1}) - \bfa{y}^{(i)}_k\right|.
\end{equation*}




Although the original GFINNs framework \cite{zhang2022gfinns} does not introduce the weighted loss or RBA, we apply them in Section \ref{sec:results} for fair comparison with WGFINNs.

\section{Analysis: Parameter estimation}\label{sec:analysis}
In this section, we present a mathematical analysis of learning a parameter in a one-dimensional dynamics from noisy trajectory data.
The focus is on identifying quantitative differences between the estimators from the strong and the weak form losses, and revealing a fundamental limitation (if any) of learning a dynamical system under noise.

For simplicity of discussion, let us consider a linear dynamical system of 
$\dot{x}=\lambda x$ with $\lambda \ne 0$ whose analytical solution is $x(t) = x(0)e^{\lambda t}$.
Since the univariate case is under consideration, a majority of notations used in this section will not use boldface.
For example, $\bfa{x} \to x$ and $\bfa{f}(\bfa{x}) \to f(x):=\lambda x$.

\subsection{Strong form loss}
Suppose that we are given a noisy trajectory $\{y_k\}_{k=0}^K$ where $y_k = x_k + \epsilon_k$, $x_k = x(k\Delta t)$ and $\epsilon_k$'s are independent random noise with mean zero and variance $\sigma^2$.
Let $f_\text{NN}(x;\theta) = \theta x$ be the model of choice where $\theta$ is to be learned from data.
Then, the learning task is equivalent to inferring the underlying system parameter $\lambda$ from $\{y_k\}$.

Consider the dynamical system $\dot{\widehat{x}}=\theta\widehat{x}$ that corresponds to \eqref{eq:strong-NN}.
After applying the forward Euler method, the one-step prediction \eqref{def:one-step-NNprediction} from $y_{k-1}$ is given by $\widehat{x}_k(\theta) = (I+\theta\Delta t) y_{k-1}$.

The strong form loss function is then given by
\begin{equation} \label{def:1D-strong-form-loss}
    \mathcal{L}^\text{strong}(\theta) = \frac{1}{K}\sum_{k=1}^K | \widehat{x}_k(\theta) - y_k|^2.
\end{equation}
Note that the $W$-norm is not used here as the univariate case is considered.
Theorem~\ref{thm:strong-form} shows that the strong form solution can not accurately approximate the target parameter $\lambda$ and will eventually diverge to infinite as $\Delta t \to 0$.
\begin{theorem} \label{thm:strong-form}
    Suppose that $\varepsilon_i$ be i.i.d. zero-mean and variance $\sigma^2$ random variables.
    Let $\theta_\emph{\text{strong}}^*(\Delta t)$ be the optimal solution that minimizes \eqref{def:1D-strong-form-loss}.
    Then 
    \begin{equation} \label{eqn:thm-strong-limit}
        \lim_{\Delta t \to 0, K\to \infty, \Delta t K = T} \Delta t\left[\theta_\emph{\text{strong}}^*(\Delta t) - \lambda\right]  \overset{\text{a.s.}}{=} - \frac{\sigma^2}{x_0^2\frac{e^{\lambda T}-1}{2\lambda T}+\sigma^2},
    \end{equation}
    which implies 
    $|\theta_\emph{\text{strong}}^*(\Delta t) - \lambda | \overset{\text{a.s.}}{\to} \infty$.
    Furthermore, if the event $E_\text{po}$ that there exists $h >0$ such that $\theta_\emph{\text{strong}}^*(h) > \lambda$ has nonzero probability, 
    we have 
    \begin{align*}
        \text{Pr}(\exists \Delta t^* >0 \ \text{such that} \ \theta_\emph{\text{strong}}^*(\Delta t^*) = \lambda \ | E_\text{po}) = 1.
    \end{align*}
\end{theorem}

\begin{proof}
    See Appendix~\ref{app:proof_thm41}
\end{proof}

In the proof of Theorem~\ref{thm:strong-form}, particularly \eqref{app:proof_thm41}, it is shown that the error can be decomposed into two parts. One part stems from the numerical integration error of Euler's method, which vanishes as $\Delta t \to 0$.
The other term is caused by the noise and depends sensitively on $\Delta t$, which diverges as $\Delta t \to 0$, making the error of the estimator significant. 
This reveals a fundamental limitation of using the strong form loss \eqref{def:1D-strong-form-loss} in handling noisy data.
The left of Figure~\ref{fig:strong-form-illustration} illustrates Theorem~\ref{thm:strong-form}, which presents the scattered errors of the strong form solution at varying $\Delta t$ and $\sigma$.
It is clearly seen that the errors become significantly large as $\Delta t$ gets smaller.
At the same time, it is observed that despite of the divergent tendency of the error, there seems to exist, at each noise level, a peculiar $\Delta t$ on which the error is (almost) zero.
As a matter of fact, this is expected from Theorem~\ref{thm:strong-form}.
On the right, the stepsize multiplied absolute errors are shown with respect to the stepsize at the noise level $10^{-2}$.
The red and blue marks represent the cases where $\theta^* > \lambda$ and $\theta^* < \lambda$, respectively. 
It is observed that there exists a stepszie that yields $\theta^*=\lambda$ and such stepsize exists when the sign changes, which numerically verifies Theorem~\ref{thm:strong-form}.
\begin{figure}[h!]
    \centering
    {\includegraphics[width=0.49\textwidth]{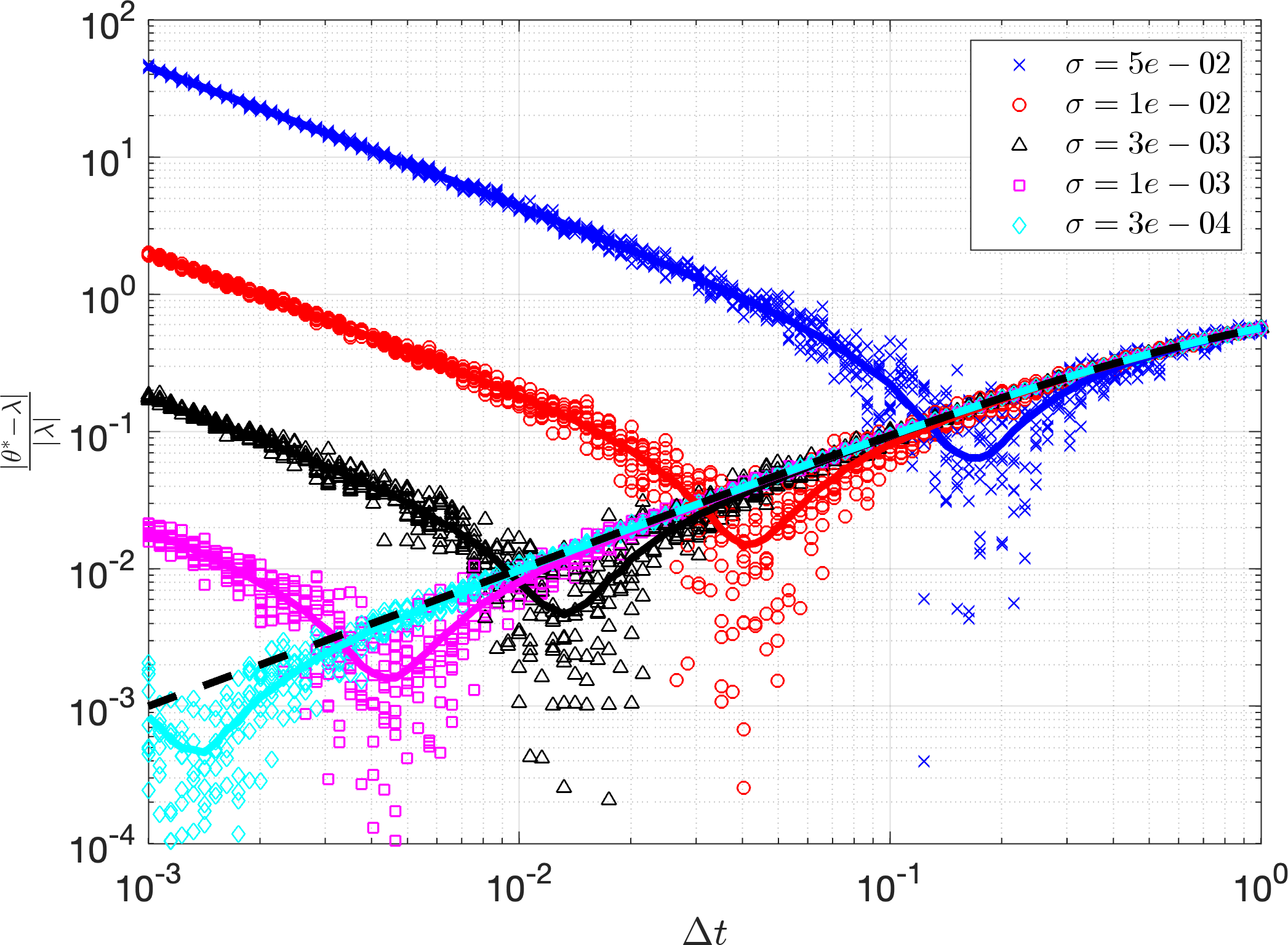}
    \includegraphics[width=0.49\textwidth]{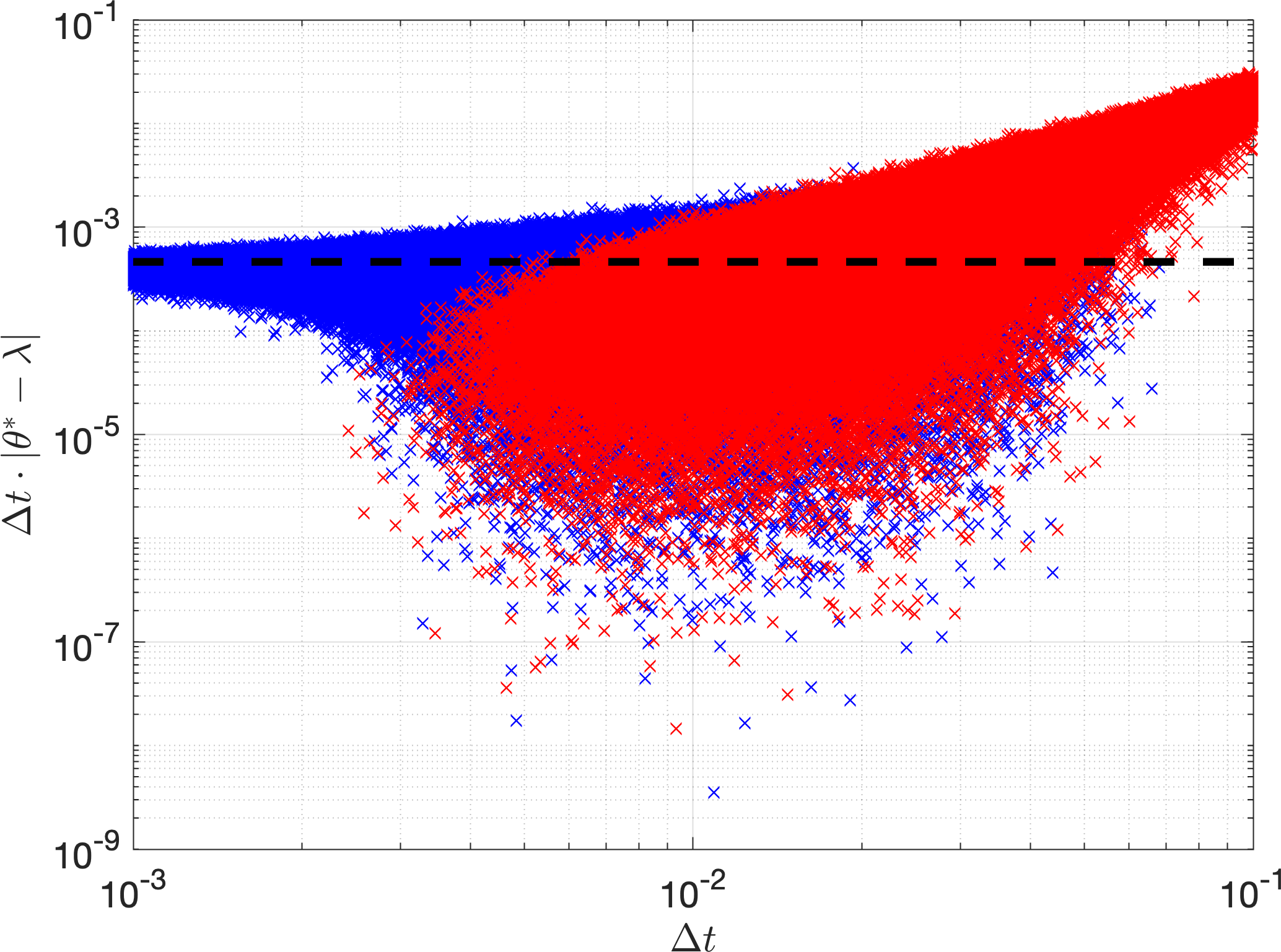}}
    \caption{(Left) The relative errors of the strong form solution with respect to the step size $\Delta t$ at varying noise levels $\sigma$. 
    The target parameter is $\lambda = -2$ and the noises are drawn from the Gaussian distributions.
    At each $\Delta t$ and $\sigma$, the scattered marks correspond to 10 independent numerical simulations.
    The solid lines represent the mean errors over 1,000 independent simulations.
    The dashed-line is the numerical integration error.
    (Right) The stepsize multiplied absolute errors with respect to the stepsize.
    The dashed-line is the limit value from \eqref{eqn:thm-strong-limit} at $\sigma=10^{-2}$.
    There exists a stepsize that yields the exact estimation of the strong form solution.
    }
    \label{fig:strong-form-illustration}
\end{figure}

\subsection{Weak form loss}
Let $\phi:[0,T]\to \mathbb{R}$ be a smooth test function supported on $[t_0, t_f] \subset [0,T]$. Let $S = t_f-t_0$ be the length of the support and consider a symmetric quadrature rule centered around $t^* = \frac{t_0+t_f}{2}$.
That is, $\{(t^*_i, w_i)\}_{i=-m}^{m}$ is a quadrature rule where $t^*_i = t^* + i h$ and $w_i = w_{-i}$ where $h = \frac{S}{2m}$, satisfying $\sum_{i=-m}^m w_i = S$.
Note that $h$ is not necessarily the same as $\Delta t$ used in the strong formulation.
For the sake of discussion, we assume 
$t^* = n_0\Delta t$ and 
$h = \bar{n}\Delta t$ for some positive integers $n_0, \bar{n}$, and let $\tilde{y}_i := y_{n_0+\bar{n}i}$ be the noisy measurement at time $t_i^*$.
Note that $y_k:= x(k\Delta t) + \varepsilon_k$ is the noisy measurement at time $k\Delta t$.

It then can be checked that the weak form loss \eqref{def:weak-form-loss} is given by
\begin{equation} \label{def:1D-weak-form-loss}
    \mathcal{L}^\text{weak}(\theta) = \left|\theta \sum_{i=-m}^m \tilde{y}_i\psi_i + \sum_{i=-m}^m \tilde{y}_i \psi'_i\right|^2,
\end{equation}
where $\phi_i = \phi(t_i^*)$, $\phi'_i = \phi'(t_i^*)$,
$\psi_{i} = w_i\phi_i$ and $\psi_i' = w_i\phi'_i$.

\begin{theorem} \label{thm:weak-form}
    Suppose that the test function satisfies 
    \begin{enumerate}
        \item Symmetric centered at $t^*$; $\phi_i = \phi_{-i}$.
        \item Asymmetric derivative; $-\phi'_i = \phi'_{-i}$.
        \item $\phi'_0 = 0$, and $\sum_{i=-m}^m \phi_iw_i = 1$.
        \item $\sum_{i=1}^m 2i\Delta t\psi_i' = -1$,
        \item $\sum_{i=1}^m \left(\psi_i'e_\lambda(ih)  + \psi_i e'_\lambda(ih)\right) = 0$,
    \end{enumerate}
    where $e_\lambda(t) = \sinh(\lambda t) - \lambda t$
    and $e'_\lambda(t) = \frac{d}{dt}e_\lambda(t) = \lambda \cosh(\lambda t)-\lambda$.
    Let $\theta^*_\emph{\text{weak}}$ be the minimizer of \eqref{def:1D-weak-form-loss}. 
    Then, 
    \begin{equation} \label{eqn:weak-error}
        \frac{\theta^*_\emph{\text{weak}} - \lambda}{\lambda} = \sigma \cdot \frac{-\langle \bm{E}, \Psi_{1,\lambda^{-1}}\rangle
        }{x_0(1 + 2\lambda^{-1}\sum_{k=1}^{m} \psi_k e'_\lambda(k\Delta t)) + \langle \bm{E}, \Psi_{1, 0}\rangle},
    \end{equation}
    where $\bm{E}= (\sigma^{-1}\varepsilon_{n_0+\bar{n}i})_{i=-m}^m$
    and $\Psi_{a,b} = (a\psi_i + b\psi_i')_{i=-m}^m$.
\end{theorem}
\begin{proof}
    See Appendix \ref{app:proof_thm43}.
\end{proof}

Theorem~\ref{thm:weak-form} presents two conditions on which \eqref{eqn:weak-error} holds.
The expression \eqref{eqn:weak-error} indicates that the relative error of the weak form solution and the noise level $\sigma$ are linearly related. In particular, the weak form solution becomes exact in the noiseless case, i.e., $\sigma=0$. See also Figure~\ref{fig:weak-form-illustration}.
The remaining question is whether such a test function that satisfies all the conditions in Theorem~\ref{thm:weak-form} exists and when the error can be small.
Theorem~\ref{prop:weak-form-L1} provides an affirmative answer for a special case of $m=1$.

\begin{theorem} \label{prop:weak-form-L1}
    Let $m=1$ and consider any test function satisfying 
    \begin{align*}
        \phi_0 &=\frac{1}{w_0}\left[1 - 2\frac{e_\lambda(h)}{Se'_\lambda(h)}\right], \quad 
        \phi_1 = \frac{1}{w_1}\frac{e_\lambda(h)}{S e'_\lambda(h)},
        \quad 
        \phi_0' = 0, \quad 
        \phi_1' = -\frac{1}{w_1S}.
    \end{align*}
    Then, the weak form solution satisfies 
    $\lim_{S \to \infty} \theta^*_\emph{\text{weak}}(S) \overset{\text{a.s.}}{=} \lambda.$
    Furthermore, if $\varepsilon_i$'s are Gaussian, 
    $\lim_{S \to 0} |\theta^*_\emph{\text{weak}}(S) - \lambda| \overset{\text{p}}{=} \infty$.
\end{theorem}
\begin{proof}
    See Appendix \ref{app:proof_prop45}.
\end{proof}


Figure \ref{fig:weak-form-illustration} illustrates Theorem~\ref{thm:weak-form}, which shows the relative errors of the weak-form solution with respect to the length of the support $S$ (left) of the test function and the noise level $\sigma$ (right).
On the left, it is clearly observed both the divergent and the convergent behaviors of the error with respect to varying $S$. 
It is also seen that the noise level is only linearly related to the error as expected from \eqref{eqn:weak-error}.

On the right, we compare the strong and weak form solutions on the same data with $\Delta t = 0.01$ with respect to the noise levels. It is observed that the relative error of the strong form converges to the numerical error (forward Euler) as the noise level gets smaller, while the one of the weak form converges to zero. 
Also, it can be seen that the weak form solution yields a higher variance than the strong form. This stems from the fact that the weak form uses $2m+1$ data samples while the strong form uses all data. 
Overall, in this particular example, the weak form tends to yield more accurate estimation while its performance depends on the length of the support.
\begin{figure}[h!]
    \centering
    {\includegraphics[width=0.48\textwidth]{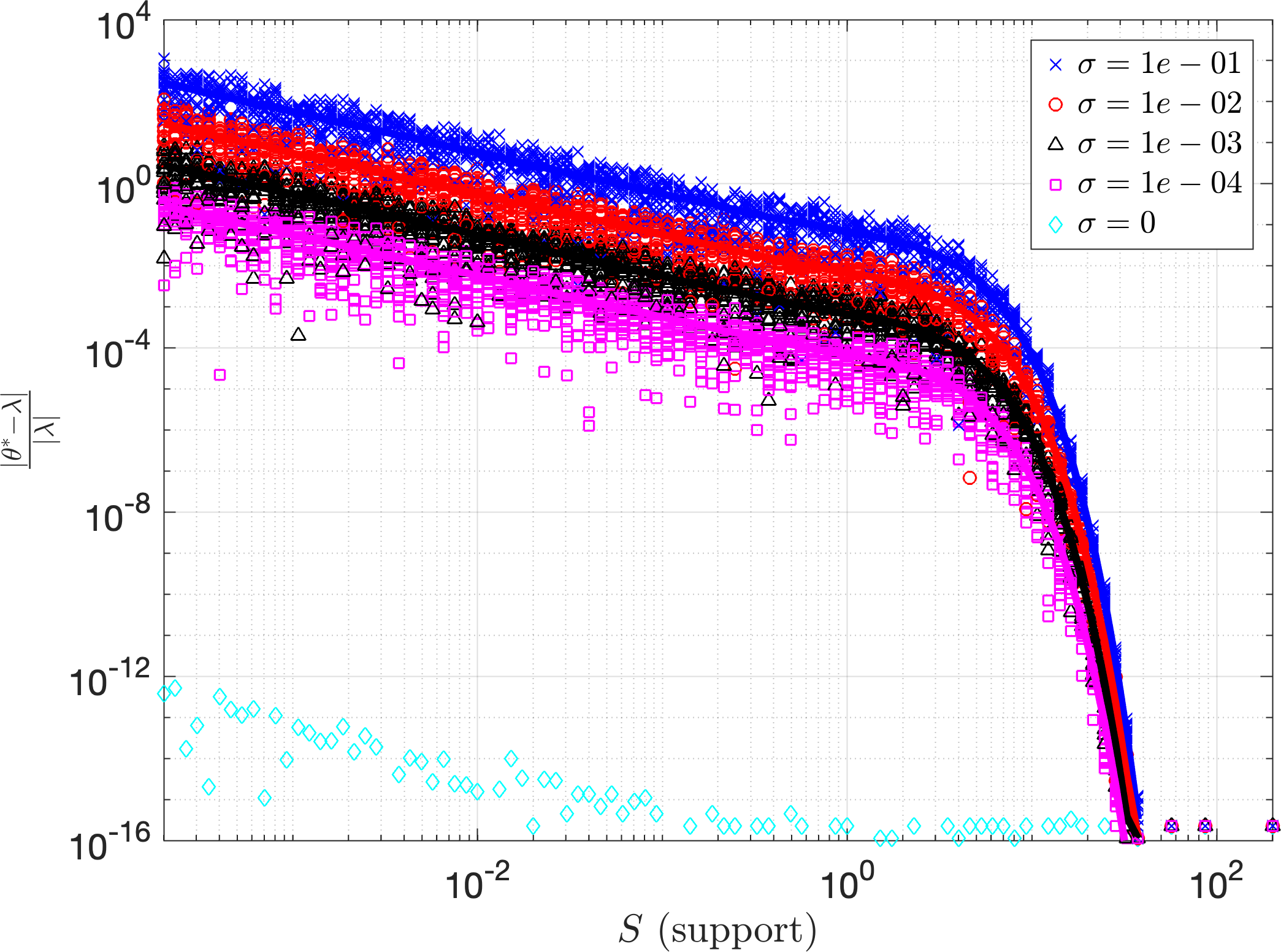}
    \includegraphics[width=0.48\textwidth]{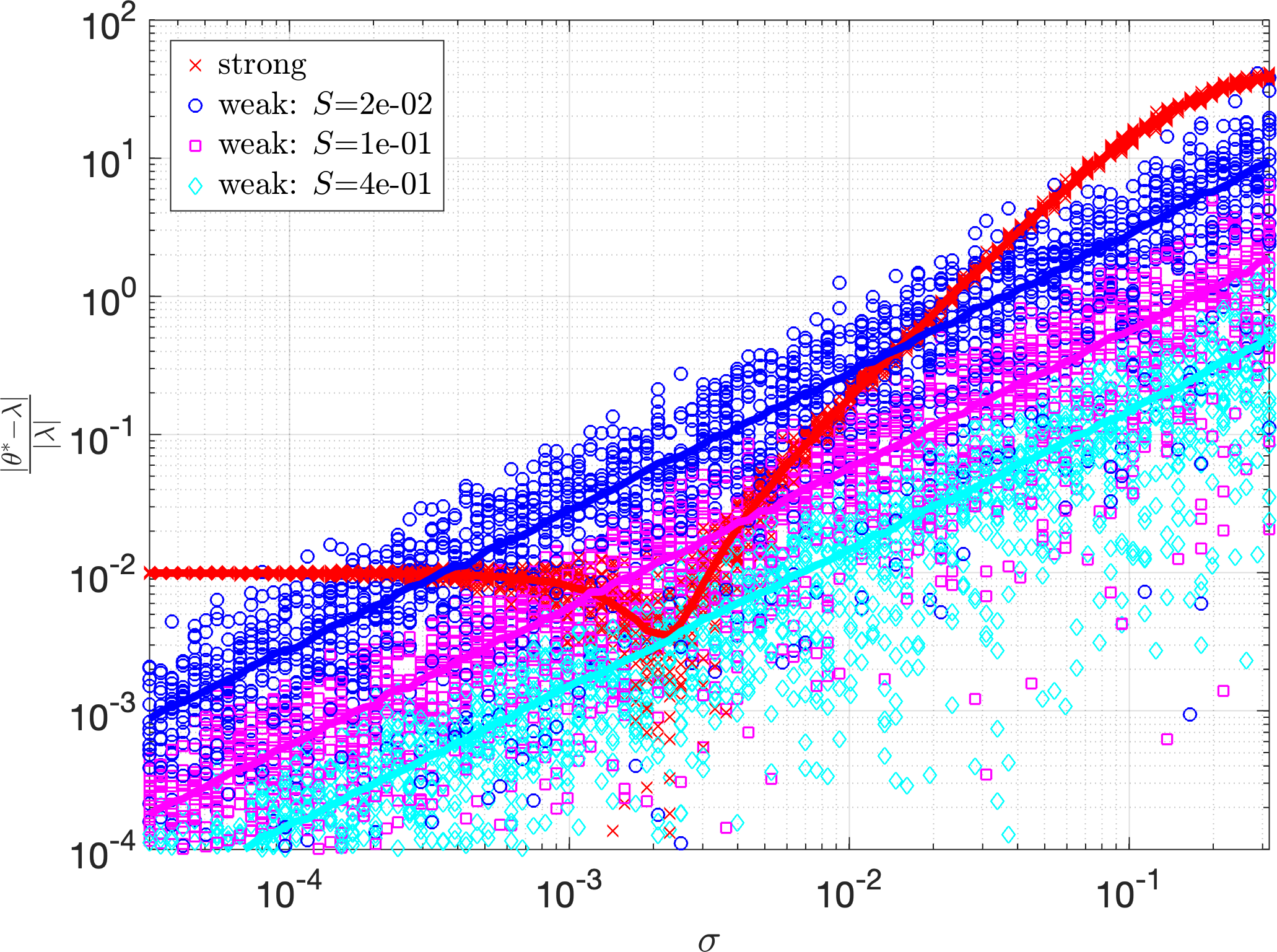}
    }
    \caption{Left: The relative errors of the weak form solutions with respect to the length of the support $S$ using the test function from Theorem~\ref{prop:weak-form-L1}, at varying different noise levels $\sigma$.
    Right: The relative errors of both the strong and weak form solutions with respect to the noise levels.
    In both figures, the scattered marks correspond to 10 independent numerical simulations at each fixed setup.
    The solid lines are the mean errors over 1,000 independent simulations.}
    \label{fig:weak-form-illustration}
\end{figure}

\section{Numerical examples}
\label{sec:results}

In this section, we present several numerical experiments to demonstrate the effectiveness of WGFINNs. 
WGFINNs and GFINNs are trained by minimizing the weighted weak-form loss function (\ref{def:weak-form-loss}) and the weighted strong-form loss (\ref{def:strong-form-loss}), respectively, with State-RBA described in Section \ref{sec:method}. To quantify the prediction accuracy, we use the following averaged relative $\ell_2$ error computed over test trajectories:

\begin{equation}
\label{eq:test_rmse}
e_{\bfa{x}}^{l_2}
= \frac{1}{d N_{\mathrm{test}}}
\sum_{i=1}^{N_{\mathrm{test}}}
\sum_{j=1}^{d}
\sqrt{
\sum_{k=0}^{K} \left(\left(\bfa{x}_{k}^{(i)}\right)_j- \left(\widetilde{\bfa{x}}_{k}^{(i)}\right)_j\right)^{2}
\Big/ 
\sum_{k=0}^{K} \left(\left(\bfa{x}_{k}^{(i)}\right)_j\right)^{2}
},
\end{equation}
where $\bfa{x}_k^{(i)} = \bfa{x}(t_k;\bfa{x}_0^{(i)})$ is the ground truth (GT) solution, $\widetilde{\bfa{x}}_k^{(i)}$ represents the corresponding prediction obtained by solving (\ref{eq:strong-NN}) with $\widehat{\bfa{x}}(0) = \bfa{x}_0^{(i)}$ by applying a numerical integrator (e.g. Runge-Kutta methods), and $N_{\text{test}}$ denotes the number of test trajectories. 

To evaluate robustness under noisy data, we corrupt training trajectories with state-dependent Gaussian noise. That is, the noisy training data are generated as
$$\bfa{y}_k^{(i)} = \bfa{x}_k^{(i)} + \bfa{\epsilon}_k^{(i)}, \quad \bfa{\epsilon}_k^{(i)} \sim \mathcal{N}(0, \text{diag}((\rho\sigma_1)^2, \ldots, (\rho\sigma_d)^2)),$$
with $\sigma_l$ defined in Section \ref{sec:sw_rba}, and $\rho \in [0,1]$ represents the noise level (e.g., $\rho = 0.05$ corresponds to $5\%$ noise).

All the implementations were done on a Livermore Computing Tuolumne system's CDNA 3 GPU (APU: AMD MI300A), located at the Lawrence Livermore National Laboratory. The system comprises a total of 4,608 GPUs, with four GPUs deployed per compute node. Each GPU provides 512 GiB of global memory.
The source codes are written in the open-source PyTorch \cite{paszke2019pytorch} and are published in GitHub\footnote{GitHub page: \url{https://github.com/pjss1223/WGFINNs}}.
Other implementation details can be found in Appendix~\ref{appendix:training_details}. 
In all the examples, WGFINNs and GFINNs share the same NN architectures and they are both trained with \texttt{AdamW} optimizer \cite{loshchilov2017decoupled}.

\subsection{Two gas containers exchanging heat and volume}
\label{sec:gas_containers}
We consider a system consisting of two ideal gas containers separated by a movable wall,
through which both heat and volume can be exchanged. 
The state of the system is characterized by four variables:
the position \(q\) and momentum \(p\) of the wall,
together with the entropies of the gases in the two containers,
denoted by \(S_1\) and \(S_2\). 
This system admits a natural formulation within the GENERIC framework and has been extensively studied in the literature~\cite{ottinger2005beyond, shang2020structure, zhang2022gfinns, park2024tlasdi}. We refer the reader to, e.g., \cite{zhang2022gfinns} for the governing equations and system parameter settings.

Figure~\ref{fig:GC_noise_err} presents the relative $\ell_2$ test errors (\ref{eq:test_rmse}) with respect to varying noise levels for both WGFINNs and GFINNs. The results of five independent simulations are shown.
It is observed that while the test errors for both methods tend to increase with the noise level, GFINNs exhibit significantly larger errors, especially at higher noise levels. This suggests that the strong form loss is more sensitive to noise in the training data. 
On the other hand, WGFINNs consistently achieve significantly lower test errors across all noise regimes and exhibit robustness to noise.

\begin{figure}[!t]
 \centering
\includegraphics[scale=0.34]{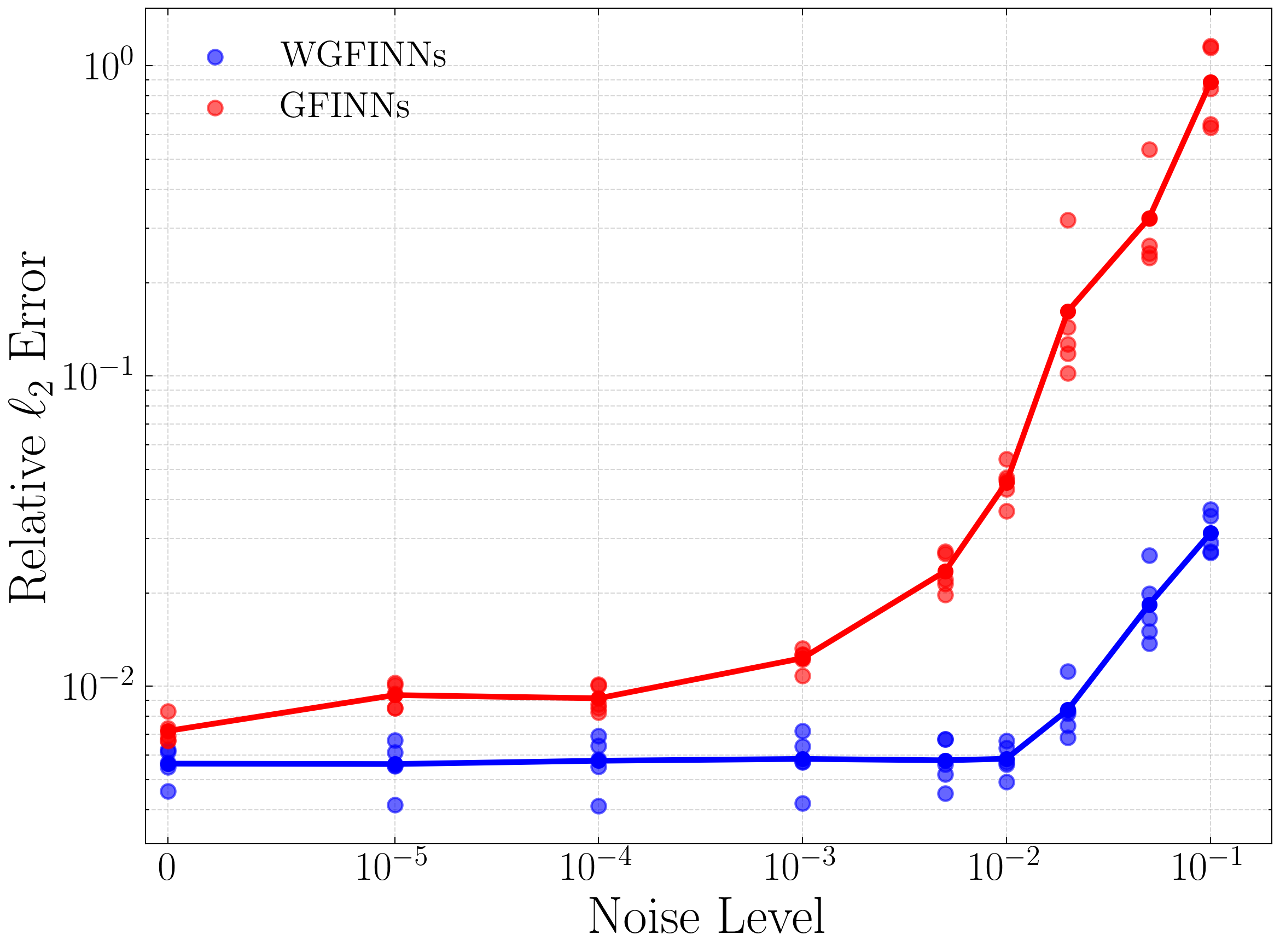}
\caption{Example \ref{sec:gas_containers}. The relative $\ell_2$ test errors (\ref{eq:test_rmse}) under varying noise levels. The solid line represents the mean across five independent simulations.}
\label{fig:GC_noise_err}
\end{figure}  



In Figure~\ref{figure:GC_WGFINNs_vs_GFINNs}, we plot trajectories of learned dynamics starting at a randomly chosen test initial state.
Each row corresponds to the results for the noise level of $0\%$ (top), $5\%$ noise (middle), and $10\%$ (bottom).
It is seen that GFINNs fail to predict meaningful trajectories under $10\%$ noise, whereas WGFINNs yield accurate predictions.

A unique advantage of using a physically structured NN model is the interpretability. In particular, the use of GFINNs, as NN models, provides the learned energy $E_{\text{NN}}$ and entropy $S_{\text{NN}}$ functions at the end of training, which can be utilized in post-analysis.
Figure~\ref{fig:GC_Entropy_Energy} shows the calibrated contours of $E_{\text{NN}}$ and $S_{\text{NN}}$ (top and middle rows) and the calibrated $S_{\text{NN}}$ evaluated along predicted trajectories across noise levels (bottom row), comparing 
them with GT. 
While how the calibration is done is detailed in Appendix \ref{sec:app_GC}, it basically applies an affine transformation to remove the indeterminacy of $E_{\text{NN}}$ and $S_{\text{NN}}$ based on GT.
In the $10\%$-noise case, WGFINNs are capable of maintaining energy and entropy contours well aligned 
with the reference, whereas GFINNs produce severely deviated energy and entropy. This illustrates WGFINNs' robust capability in discovering physical quantities from high-level, noisy observations.
Furthermore, the bottom row illustrates compliance with the second law of thermodynamics, which all the models follow well. Yet, the one by GFINNs differs significantly from GT, whereas WGFINNs produce accurate predictions.


   \begin{figure}[!ht]
	\centering
 {\includegraphics[width=\textwidth]{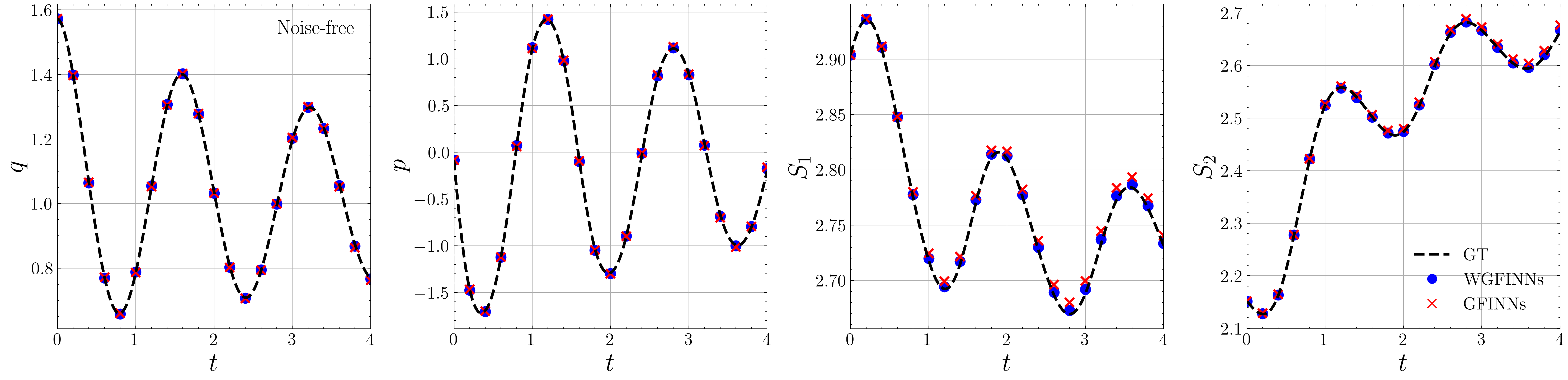}
 \includegraphics[width=\textwidth]{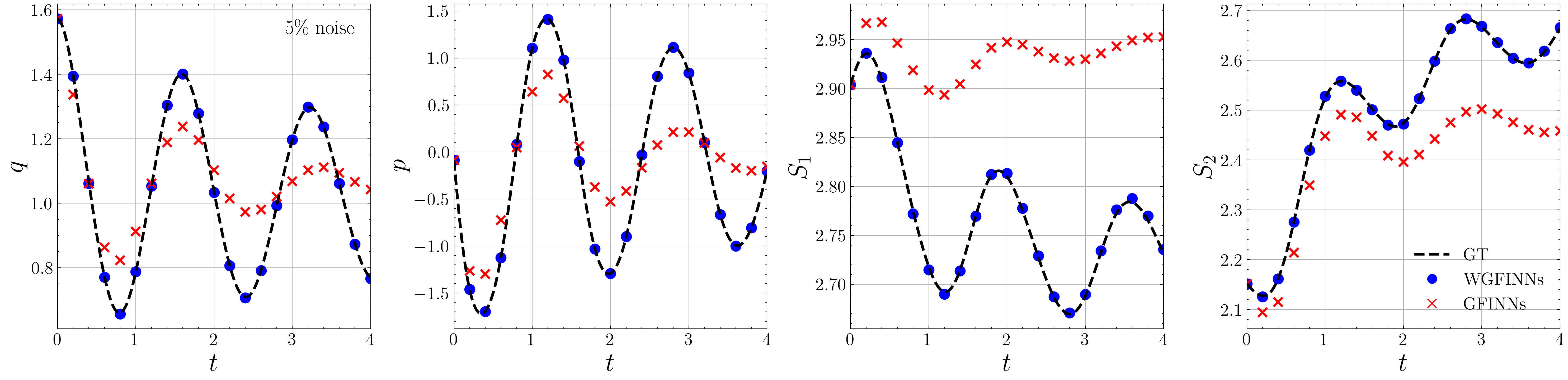}
 \includegraphics[width=\textwidth]{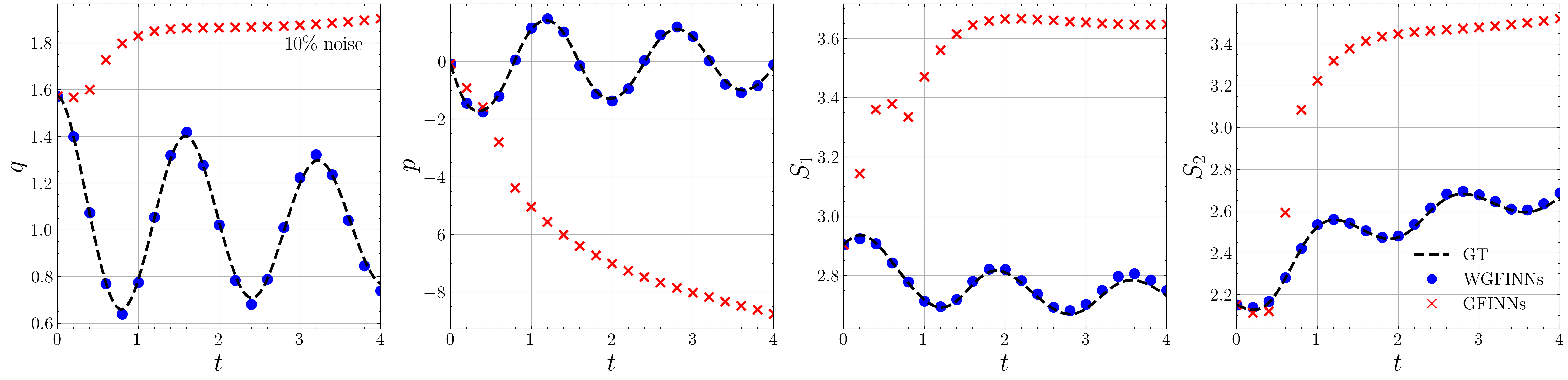}
 }
 \vspace{-2ex}
  \caption{Example \ref{sec:gas_containers}. Ground-truth (GT) trajectories and the corresponding predictions by WGFINNs and GFINNs.
From top to bottom: noise-free data, 5$\%$ noise, and 10$\%$ noise.
For each method, the results obtained from the model with the smallest training loss among five independent runs are shown.}
\label{figure:GC_WGFINNs_vs_GFINNs}
\end{figure}


\begin{figure}[ht]
 \centering
 {\includegraphics[width=\textwidth]{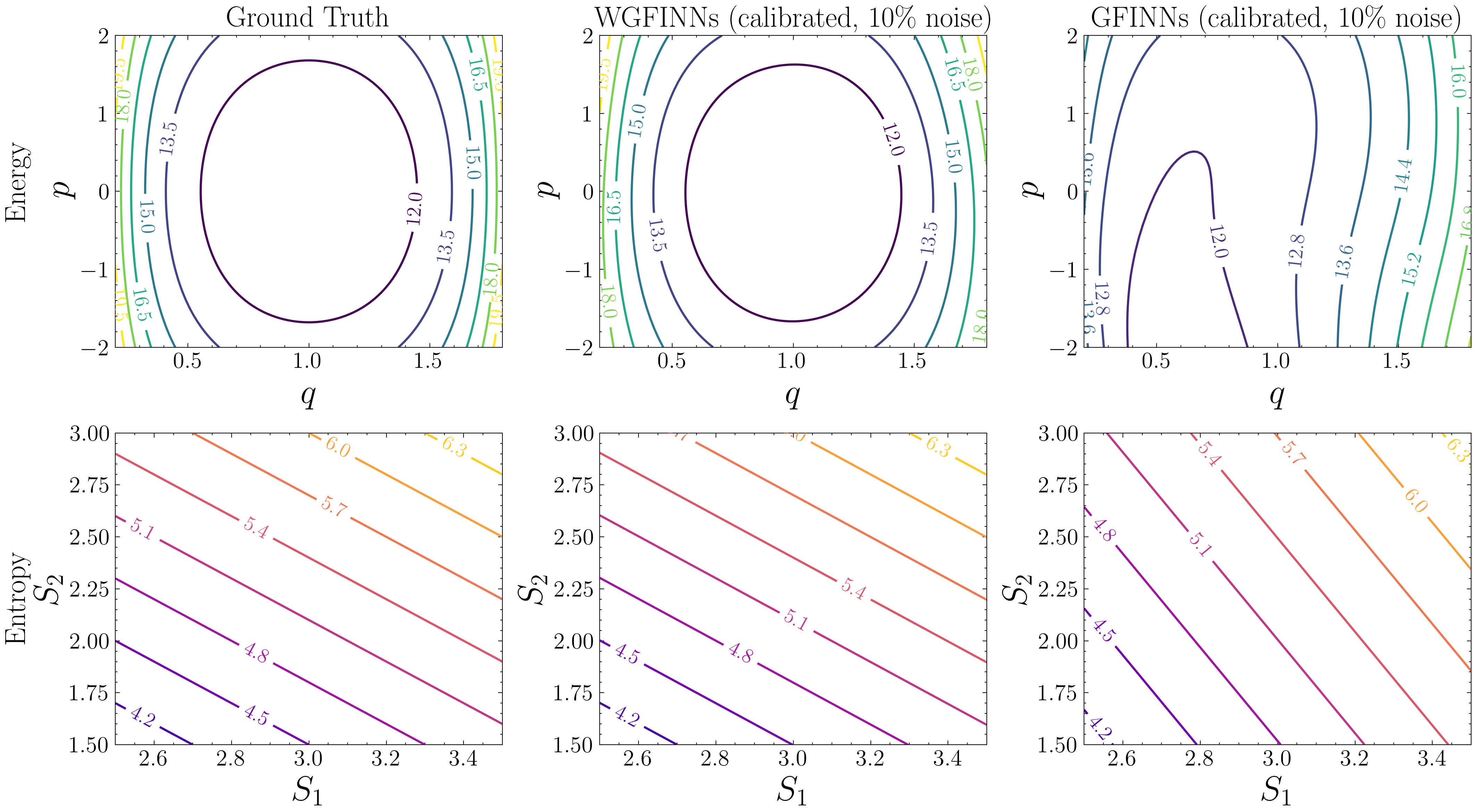}
 \hspace*{0.007\textwidth}
 \includegraphics[width=0.99\textwidth]{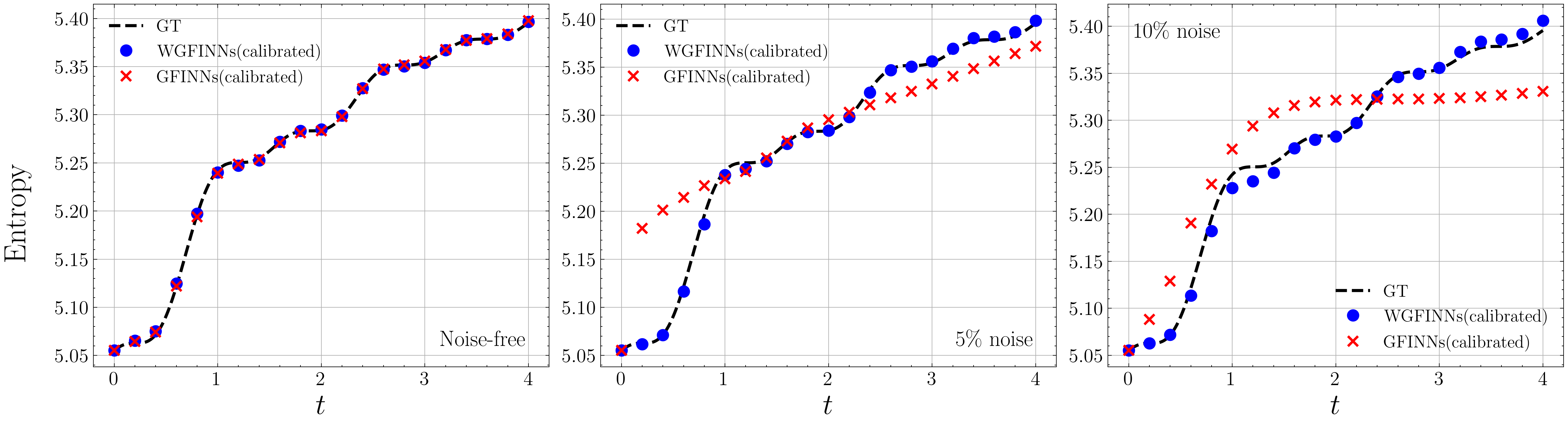}
 }
 \caption{Example \ref{sec:gas_containers}. Learned energy and entropy: 
(Top) Energy contours in the $q$-$p$ plane ($S_1 = S_2 = 2.5$ fixed). (Middle) Entropy contours in the $S_1$-$S_2$ plane ($q=1, p=0$ fixed). 
(Bottom) Calibrated entropy $S_{\mathrm{NN}}$ along a predicted trajectory for noise levels $0\%$, $5\%$, and $10\%$. 
All learned quantities are calibrated via least squares.}
\label{fig:GC_Entropy_Energy}
\end{figure}  

\subsection{Thermoelastic double pendulum} \label{sec:double_p}
We consider a two-dimensional thermoelastic double pendulum system \cite{romero2009thermodynamically, zhang2022gfinns}.
The state of the system is characterized by six variables:
the positions $\bfa{q}_1, \bfa{q}_2 \in \mathbb{R}^2$ and momenta $\bfa{p}_1, \bfa{p}_2 \in \mathbb{R}^2$
of the two masses, together with the entropies of the two springs,
denoted by $S_1$ and $S_2$, yielding a total state dimension of $d = 10$.
The detailed governing equations and system parameter settings can be found, e.g., in~\cite{zhang2022gfinns}.


Figure~\ref{fig:DP_noise_err} shows the relative $\ell_2$ test errors (\ref{eq:test_rmse}) of five independent simulations with respect to noise levels for both WGFINNs and GFINNs. 
As in the previous example, the test errors of both methods increase with noise level, and GFINNs incur substantially larger errors at high noise levels, again indicating the strong-form loss's sensitivity to noise. Yet, WGFINNs maintain consistently lower test errors across all noise regimes, demonstrating robust performance under data corruption.

\begin{figure}[ht]
 \centering
\includegraphics[scale=0.34]{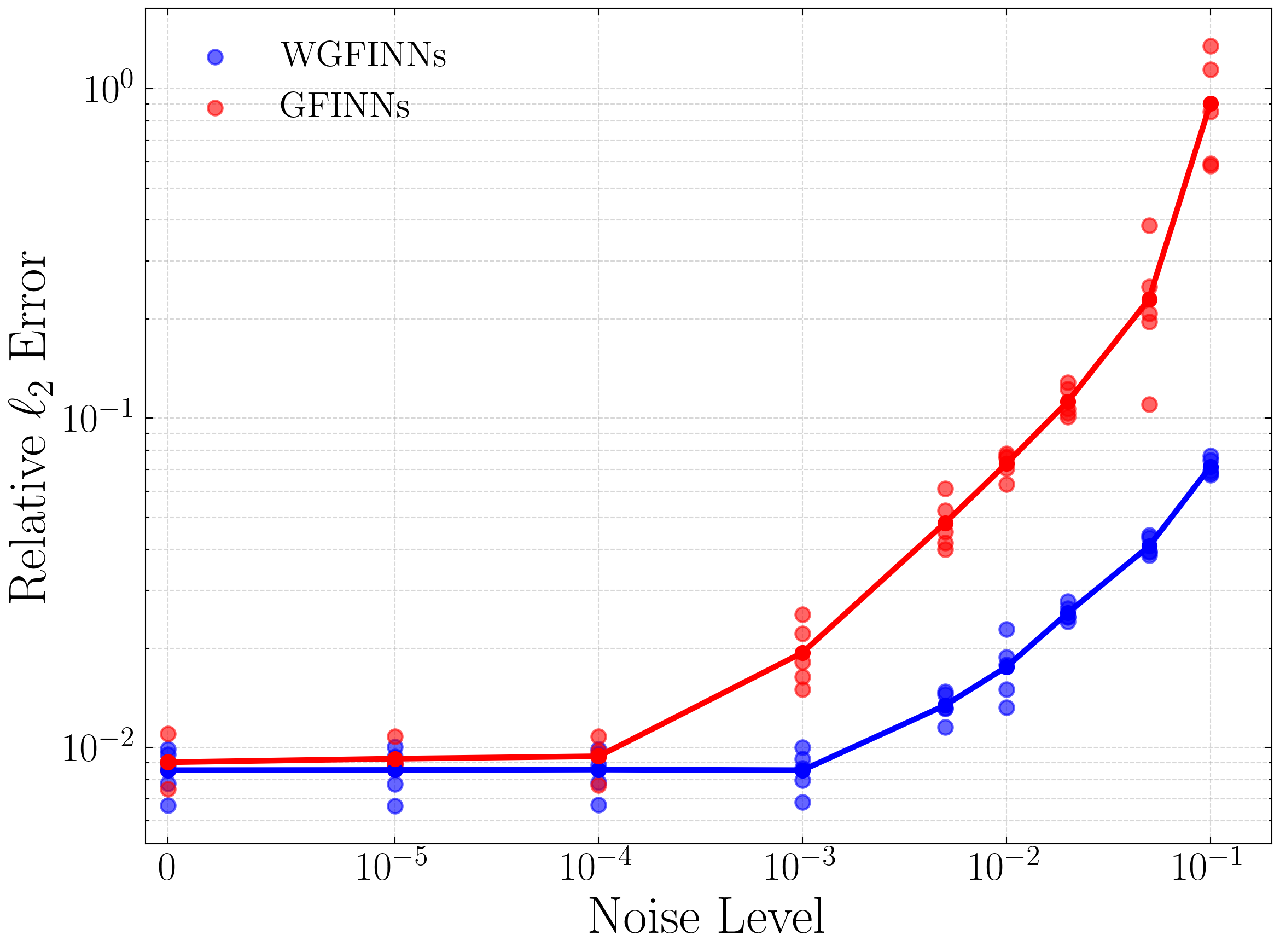}
\caption{Example \ref{sec:double_p}. The relative $\ell_2$ test errors (\ref{eq:test_rmse}) under varying noise levels. The solid line represents the mean across five independent simulations.}
\label{fig:DP_noise_err}
\end{figure}  

The trajectories of learned thermodynamics starting at a randomly chosen test initial state are shown in Figure~\ref{figure:DP_WGFINNs_vs_GFINNs_prediction}. The states of interest are the lengths of the springs, defined by
$\lambda_1 = \|\bfa{q}_1\|$ and $\lambda_2 = \|\bfa{q}_2 - \bfa{q}_1\|$, and the total entropy $S$ of the system, defined by $S=S_1 + S_2$.
Each row represents the results for a different noise level: $0\%$ (top), $5\%$ (middle), and $10\%$ (bottom).
It is again observed that, while both approaches closely follow the reference dynamics in the noise-free case, in the presence of noise, GFINNs' predictions deviate substantially from GT, whereas WGFINNs maintain close agreement throughout. 
Under $10\%$ noise, GFINNs are unable to capture any meaningful dynamics, while WGFINNs successfully track the 
reference dynamics.



  \begin{figure}[!ht]
	\centering
 {\includegraphics[width=\textwidth]{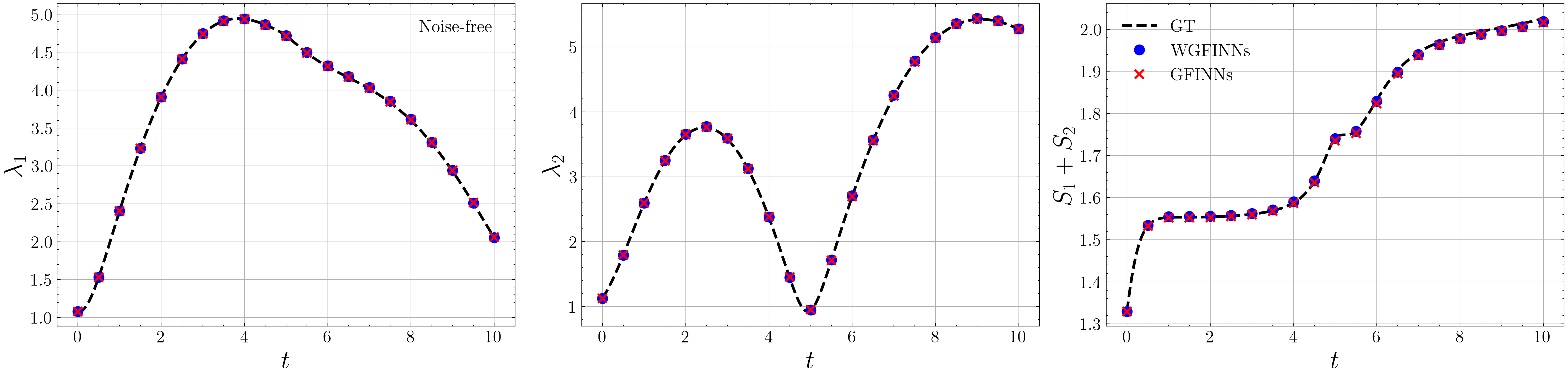}
 \includegraphics[width=\textwidth]{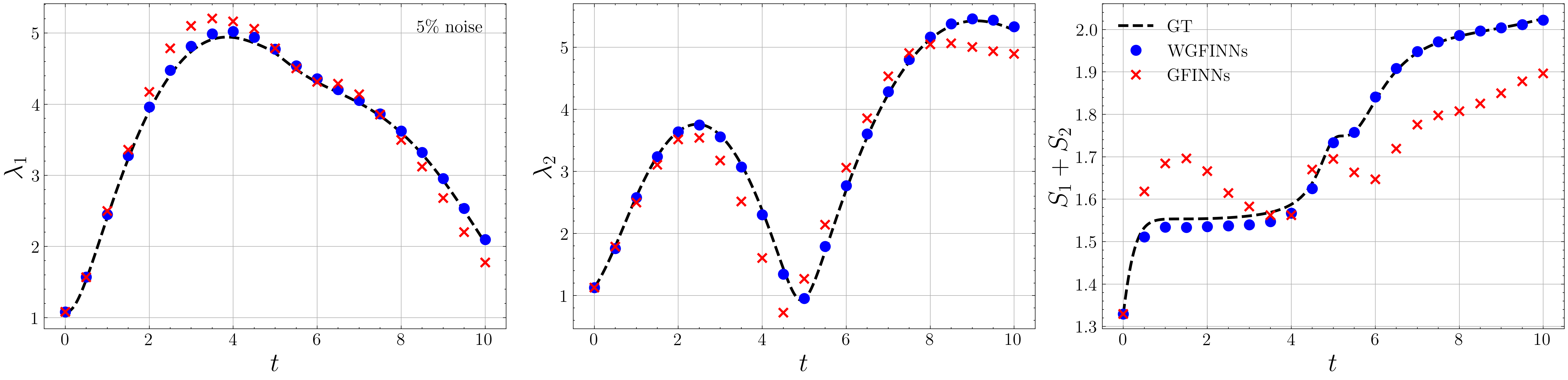}
 \includegraphics[width=\textwidth]{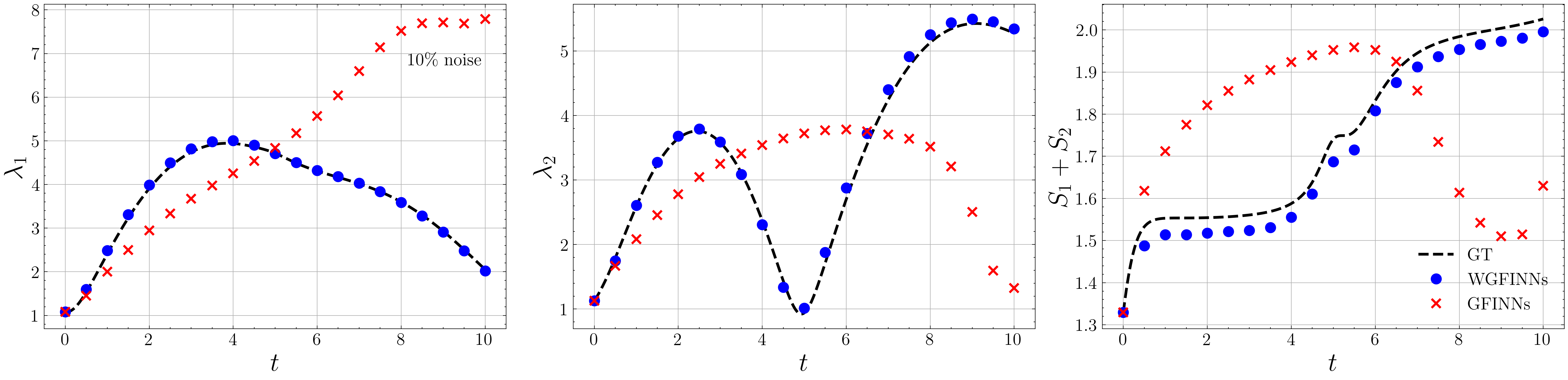}
 }
  \caption{Example \ref{sec:double_p}. Ground-truth (GT) trajectories and the corresponding predictions by WGFINNs and GFINNs.
From top to bottom: noise-free data, 5$\%$ noise, and 10$\%$ noise.
For each method, the results obtained from the model with the smallest training loss among five independent runs are shown.}
\label{figure:DP_WGFINNs_vs_GFINNs_prediction}
\end{figure}

\subsection{Linearly damped system}
\label{sec:linearly_damped}
Let us consider a linearly damped system \cite{shang2020structure} that exhibits a natural GENERIC framework characterized by the three state variables: the particle’s position \( q \), momentum \( p \), and entropy \( S \) of the surrounding thermal reservoir. 
The detailed governing equations and the system parameter choices can be found in~\cite{shang2020structure}.

\begin{figure}[ht]
 \centering
\includegraphics[scale=0.34]{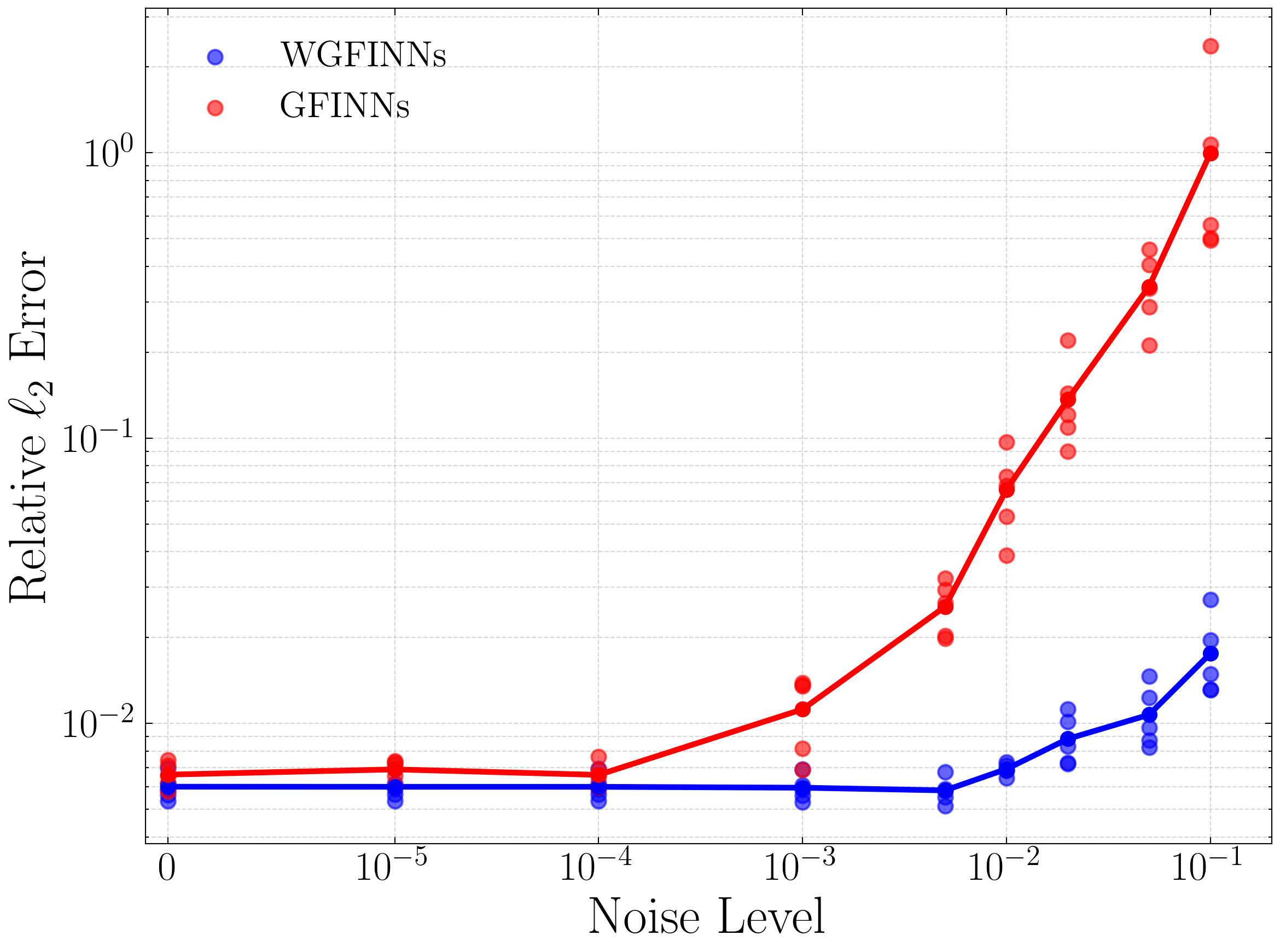}
\caption{Example \ref{sec:linearly_damped}. The relative $\ell_2$ test errors (\ref{eq:test_rmse}) under varying noise levels. The solid line represents the mean across five independent simulations.}
\label{fig:LD_noise_err}
\end{figure}  

Similar to the previous two examples, we present the relative $\ell_2$ test errors from five independent runs with respect to the noise levels for WGFINNs and GFINNs in Figure~\ref{fig:LD_noise_err}. 
Again, it is consistently observed that (1) the test errors of both methods increase as the noise level grows, (2) GFINNs produce significantly larger errors, especially at high noise levels, and (3) WGFINNs remain accurate across all noise regimes.
The trajectories by learned dynamics are shown in Figure~\ref{figure:LD_WGFINNs_vs_GFINNs_prediction}.
Again, both methods remain accurate in the absence of noise, while WGFINNs outperform GFINNs in the presence of noise.

  \begin{figure}[!ht]
	\centering
 {\includegraphics[width=\textwidth]{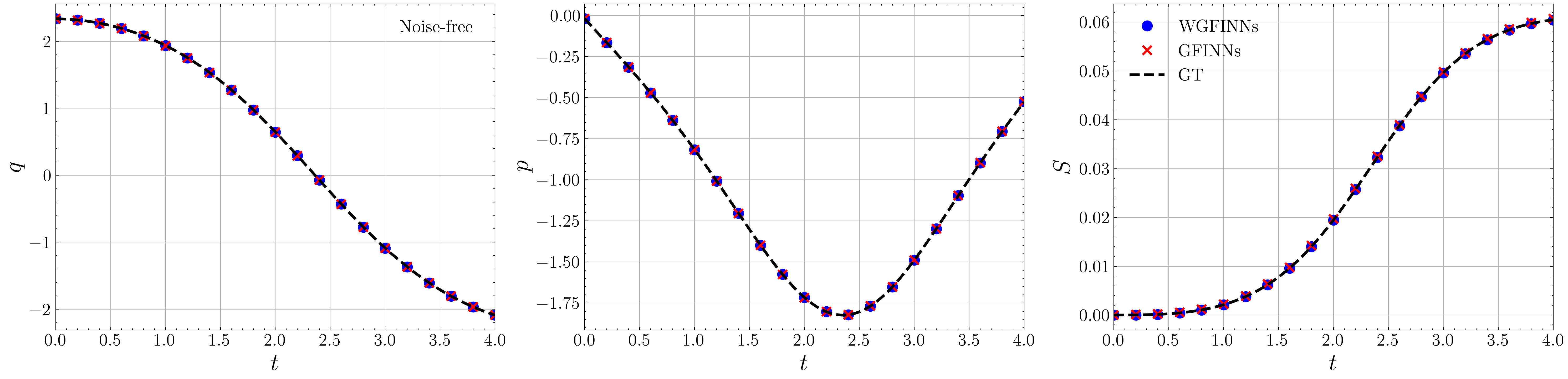}
 \includegraphics[width=\textwidth]{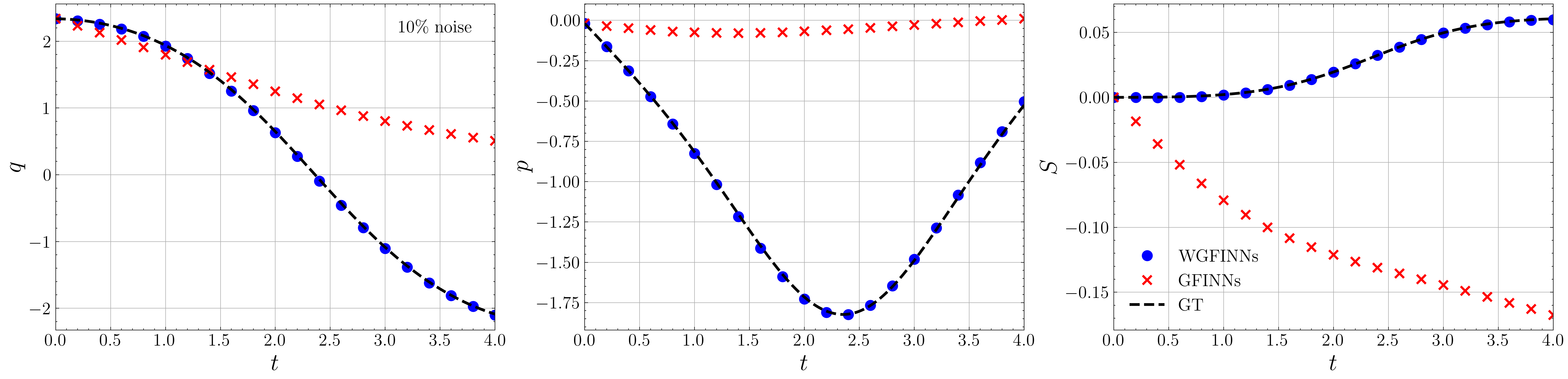}
 }
  \caption{Example \ref{sec:linearly_damped}. GT trajectories and the corresponding predictions by WGFINNs and GFINNs.
(Top) noise-free, i.e., 0$\%$ noise and (Bottom) 10$\%$ noise. See Figure~\ref{fig:LD_WGFINNs_vs_GFINNs_noise_example} for 5$\%$ noise.
For each method, the results obtained from the model with the smallest training loss among five independent runs are shown.}
\label{figure:LD_WGFINNs_vs_GFINNs_prediction}
\end{figure}


\subsection{1D/1V Vlasov–Poisson equation: electric field dynamics}
\label{sec:plasma}
Let us consider a parametric 1D/1V Vlasov--Poisson 
equation~\cite{he2024wglasdi}, which models the collisionless evolution of a plasma 
distribution function $f(x,v)$ under a self-consistent electrostatic potential $\Phi(t,x)$. We refer to~\cite{he2024wglasdi} for the full governing system defined over $t\in[8,10]$, $x\in [0,2\pi]$, $v\in[-7,7]$.
The governing dynamics depend on two physical parameters, the temperature $T$ 
and the wavenumber $k$, collected as 
$\boldsymbol{\mu} = (T, k) \in \mathcal{D} = [0.9, 1.1] \times [1.0, 1.2]$.
The electric field is given by $E(t,x) = -\partial \Phi(t,x) / \partial x$.
We consider the electric field evaluated at uniformly spaced spatial 
locations as state variables. Specifically, we define 
$E_i(t):= E(t, (2\pi/9)\, i)$ for $i = 1, \dots, 8$, and learn the dynamics 
of the resulting eight-dimensional state vector.


Figure~\ref{fig:PL_noise_err} presents the relative $\ell_2$ test errors of five independent simulations with respect to noise levels for WGFINNs and 
GFINNs. The two methods exhibit comparable performance at low noise levels, but GFINNs produce significantly larger errors as the noise level increases. Yet, WGFINNs remain accurate in high-noise regimes, again confirming their robustness to noise. 

\begin{figure}[ht]
 \centering
\includegraphics[scale=0.34]{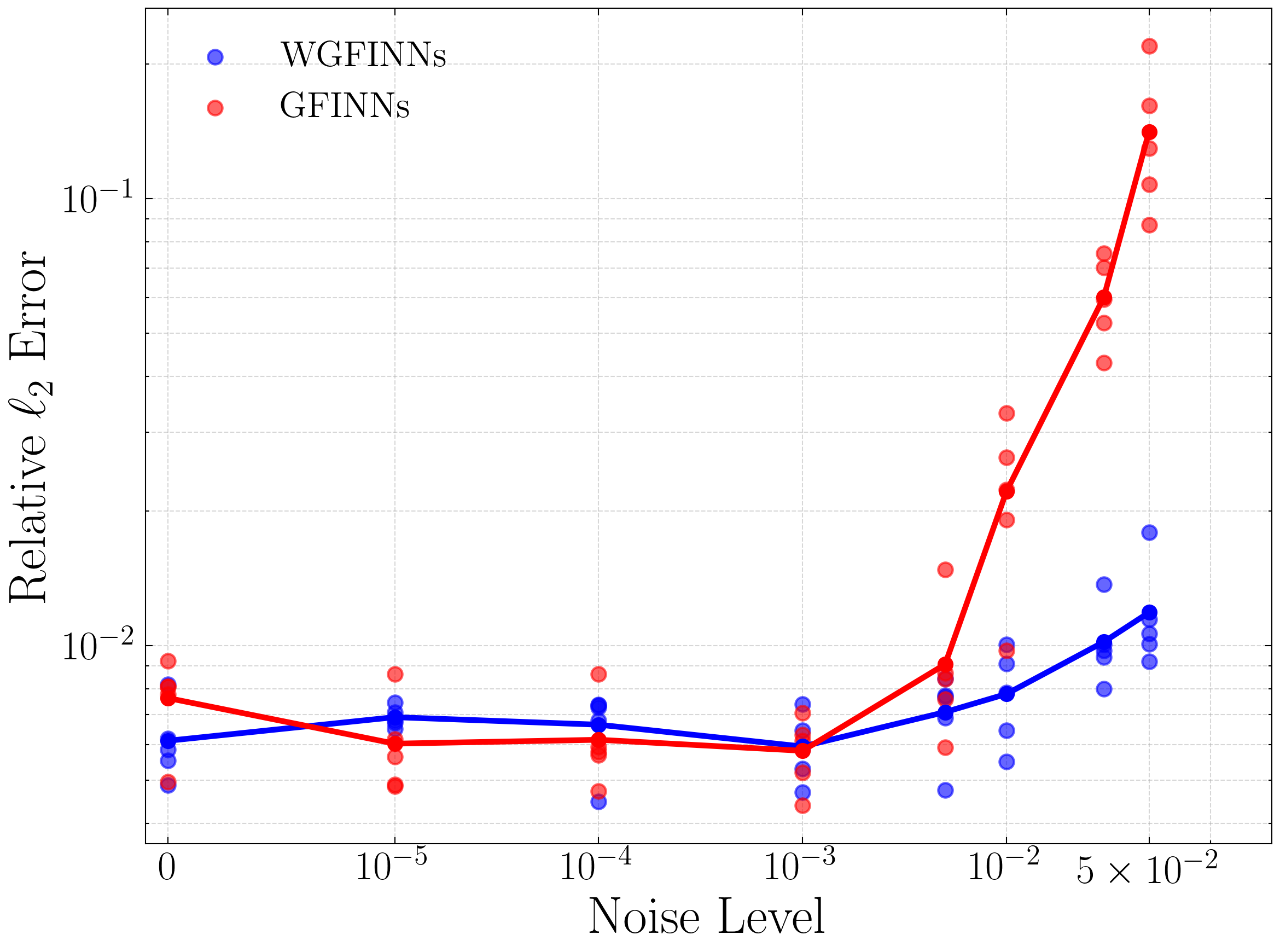}
\caption{Example \ref{sec:plasma}. The relative $\ell_2$ test errors (\ref{eq:test_rmse}) with respect to noise levels. The solid line represents the mean across five independent simulations.}
\label{fig:PL_noise_err}
\end{figure}  

Figure~\ref{figure:PL_WGFINNs_vs_GFINNs_prediction} depicts the predicted trajectories of $E_1$ and $E_2$ starting from a randomly selected initial state at varying noise levels. Again, it is seen that WGFINNs outperform GFINNs, especially in the presence of noise.



\begin{figure}[!t]
	\centering
 {\includegraphics[width=\textwidth]{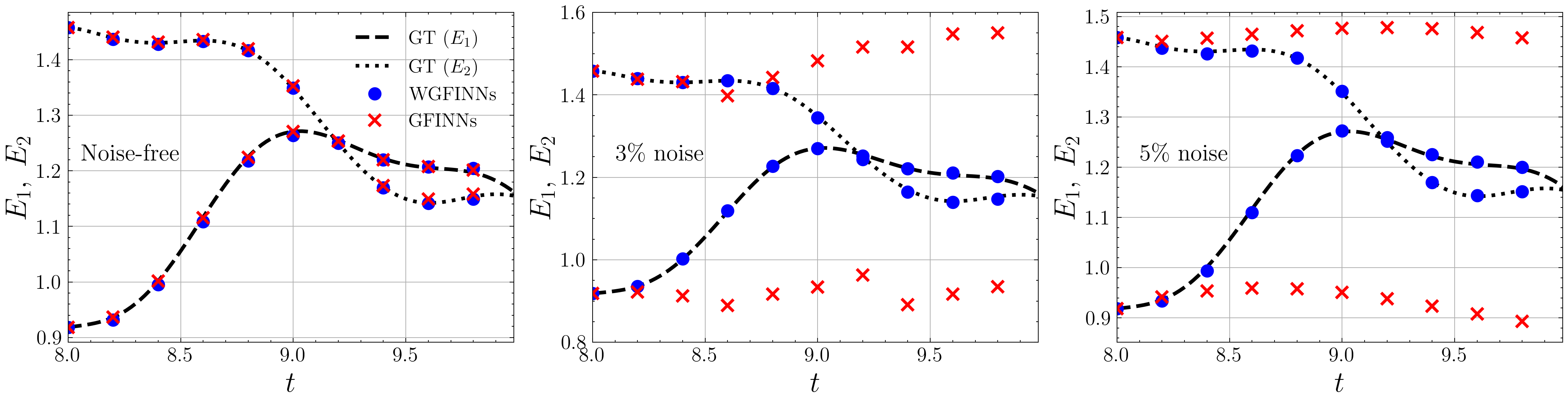}}
  \caption{Example \ref{sec:plasma}. GT states $E_1$, $E_2$, and the corresponding predictions by WGFINNs and GFINNs.
From left to right: $0\%$ noise (noise-free), $3\%$ noise, and $5\%$ noise.
For each method, the results obtained from the model with the median training loss among five independent runs are shown.}
\label{figure:PL_WGFINNs_vs_GFINNs_prediction}
\end{figure}
\section{Conclusion}
\label{sec:conclusion}

We developed a learning framework that enables robust learning in the presence of noise in training data. 
In particular, we focused on using the GENERIC formalism-informed NNs (GFINNs) to exploit its thermodynamic structures, and proposed the Weak formulation-based GFINNs, namely, WGFINNs, that integrate a weak-form-based loss function with GFINNs. 
The learning framework introduces (1) state-wise weighted norm and (2) state-wise RBA that offer an effective training strategy that facilitates balanced learning across state variables of differing scales. 
By focusing on univariate linear dynamics, we presented a theoretical analysis showing quantitative differences between the strong-form and weak-form estimators. A major finding is that the strong form will diverge as the time step decreases in the presence of noise, while the weak-form estimator can remain accurate provided that the test functions satisfy certain conditions. 


Throughout extensive numerical experiments, we demonstrated the effectiveness of WGFINNs, which consistently outperform GFINNs across all noise levels. Due to the inherited thermodynamic structures, while both WGFINNs and GFINNs yield learned energy and entropy functions at the end of learning, WGFINNs offer significantly accurate predictions while preserving the thermodynamic laws.
Future work includes extensions to shock hydrodynamics, chaotic fluid systems, and non-autonomous identification of open systems.

\section*{Acknowledgments}
J.\ R.\ Park was partially supported by the associate member program provided by Korea Institute for Advanced Study, and by a research grant (K2432211) from the College of Public Policy at Korea University. 
J.\ R.\ Park would like to thank Dr. Xiaolong He for his help on the data generation for the example of 1D/1V Vlasov-Poisson equation.
S. W. Cheung and Y. Choi were partially supported by 
the U.S. Department of Energy, Office of Science, 
Office of Advanced Scientific Computing Research,
Scientific Discovery through Advanced Computing (SciDAC) program through
the LEADS SciDAC Institute under Project Number SCW1933 at LLNL,
and the LLNL Laboratory Directed Research and Development (LDRD) Program 
under Project No. 24-ERD-035.
LLNL is operated by Lawrence Livermore National Security LLC, for the
U.S. Department of Energy, National Nuclear Security Administration
under Contract DE-AC52-07NA27344.
IM release number: LLNL-JRNL-2016230.
Y. Shin was partially supported by the National Science Foundation under the grant DMS-2513966 (Computational Mathematics Program) and the NRF grant funded by the Ministry of Science and ICT of Korea (RS-2023-00219980).

\begin{appendices}
\clearpage
\section{Proofs of parameter estimation analysis}\label{app:proofs}
In this appendix, we provide detailed proofs for the theoretical results 
presented in Section \ref{sec:analysis}. 
\subsection{Proof of Theorem~\ref{thm:strong-form}}\label{app:proof_thm41} 

Without loss of generality, let $T=1$.
Let $K$ be the largest integer less than or equal to $h^{-1}$ so that $hK \le 1$,
and consider $\bm{P}_K(z)=(x_0e^{jz})_{j=0}^{K-1}$.
Let $\bm{E}_{0,K} = (\sigma^{-1}\varepsilon_0^{(K)},\dots,\sigma^{-1}\varepsilon_{K-1}^{(K)})$ and
$\bm{E}_{1,K} = (\sigma^{-1}\varepsilon_1^{(K)},\dots,\sigma^{-1}\varepsilon_K^{(K)})$
be the standardized noise vectors.
Here the dependency on $K$ is explicitly shown
and $\varepsilon_j^{(K)}$'s are i.i.d. random variables for every $j$ and $K$.
It can be checked that $\bm{Y}_K(\lambda h) = \bm{P}_K(\lambda h) + \sigma \bm{E}_{0,K}$

Since the optimal solution can be explicitly written as 
$\theta^*_K = \frac{\sum_{i = 1}^{K} y_{i -1} (y_i - y_{i-1})}{\Delta t \sum_{i = 1}^{K} y_{i-1}^2}$, it can be checked that 
\begin{equation} \label{eqn:strong-form-solution}
    e_K(\lambda h):=\theta^*_K - \lambda = \mathcal{E}_{\lambda,h} + \frac{\sigma}{h}\frac{G_K(\lambda h)}{\|\bm{Y}_K(\lambda h)\|^2},
\end{equation}
where $\mathcal{E}_{\lambda,t}:= \frac{e^{\lambda t} - 1}{t}-\lambda$.
and $G_K(\lambda h):= \langle \bm{Y}_K(\lambda h), \bm{E}_{1,K} - e^{\lambda h}\bm{E}_{0,K}\rangle$.

For $\frac{1}{N+1} < t \le \frac{1}{N}$ where $N$ is a positive integer,
let us consider the continuous extension of the error: $e(t)= \mathcal{E}_{\lambda,t} + \sigma \frac{g(t)}{ty(t)}$ where 
\begin{align*}
    y(t) &= \|\bm{Y}_N(\frac{\lambda}{N})\|_2^2 + \frac{\frac{1}{N}-t}{\frac{1}{N}-\frac{1}{N+1}} \left[\|\bm{Y}_{N+1}(\frac{\lambda}{N+1})\|_2^2 - \|\bm{Y}_N(\frac{\lambda}{N})\|_2^2\right], \\
    g(t) &= G_N(\frac{\lambda}{N})+ \frac{\frac{1}{N}-t}{\frac{1}{N}-\frac{1}{N+1}} 
    \left[G_{N+1}(\frac{\lambda}{N+1}) - G_N(\frac{\lambda}{N}) \right].
\end{align*}
By the strong law of large numbers and $\|\bm{P}_K(\lambda h)\|_2^2 = x_0^2\frac{e^{2\lambda hK}-1}{e^{2\lambda h}-1}$, we have
\begin{align*}
    \lim_{t \to 0} t y(t) \overset{\text{a.s.}}{=} x_0^2 \frac{e^{2\lambda}-1}{2\lambda}+ \sigma^2,
    \quad 
    \lim_{t \to 0} t g(t) \overset{\text{a.s.}}{=} -\sigma, 
    \quad 
    \lim_{t \to 0} t e(t) \overset{\text{a.s.}}{=} -\frac{\sigma^2}{x_0^2 \frac{e^{2\lambda}-1}{2\lambda}+ \sigma^2} < 0. 
\end{align*}
Hence, $\lim_{t \to 0} |e(t)|\overset{\text{a.s.}}{=} \infty$.

Let $E_{\text{po}}$ be the event that there exists $t^+ > 0$ such that $e(t^+) > 0$.
Note that since $e^x-1 \ge x$ for all $x \in \mathbb{R}$ and the equality holds only when $x=0$, we have $t\mathcal{E}_{\lambda,t} > 0$ for all $t > 0$.
Hence, assuming $\text{Pr}(E_\text{po}) > 0$, it follows from 
the almost surely continuity of $e$ and the intermediate value theorem that 
$\text{Pr}(\exists t^* > 0 \ \text{such that} \ e(t^*) = 0 \ | E_\text{po}) = 1$.

\hfill $\square$

\subsection{ Proof of Theorem~\ref{thm:weak-form}}\label{app:proof_thm43} 

We first introduce the following lemma. 

\begin{lemma} \label{lemma:Taylor}
    Let $x(t) = x_0e^{\lambda t}$ with $x_k = x(kh)$. Then,
    \begin{align*}
        \frac{x_k - x_{-k}}{2kh} &= x'_0 + x_0\frac{e_\lambda(kh)}{kh}, \\
        \frac{x_k -x_0 + x_{-k}-x_0}{kh} &=\frac{2x_0}{\lambda} \frac{e'_\lambda(kh)}{kh},
    \end{align*}
    where $e_\lambda(t) = \sinh(\lambda t) - \lambda t$
    and $e'_\lambda(t) = \frac{d}{dt}e_\lambda(t)$.
\end{lemma}
\begin{proof}
    Since $x(t) = x_0e^{\lambda t}$,
    it follows from the Taylor's theorem that 
    \begin{align*}
        \frac{x(kh) - x(0)}{kh} &= x'(0) + \frac{1}{kh}\int_0^{kh} x''(t) (kh-t) dt \\
        &= x'(0) + \frac{\lambda^2}{kh} x_0 \int_0^{kh}  e^{\lambda t} (kh-t) dt \\
        &= x'(0) + \frac{x_0}{kh}(e^{\lambda kh}-1-\lambda kh).
    \end{align*}
    Therefore,
    \begin{align*}
        \frac{x(kh) - x(0)}{kh} - \frac{x(-kh) - x(0)}{-kh}
        &= \frac{x_0}{kh}(e^{\lambda kh}-1-\lambda kh) - \frac{x_0}{-kh}(e^{-\lambda kh}-1+\lambda kh) \\
        &= \frac{2x_0}{kh}(\frac{e^{\lambda kh}+e^{-\lambda kh}}{2}-1)
        = 2\lambda x_0 \frac{\cosh(\lambda kh)-1}{\lambda kh} 
    \end{align*}
    Similarly, 
    \begin{align*}
        \frac{x(kh)-x(-kh)}{2kh} &= x'(0) + \frac{x_0}{kh}(\frac{e^{\lambda kh}-e^{-\lambda kh}}{2}-\lambda kh) \\
        &= x'(0) + \lambda x_0\frac{\sinh(\lambda kh) - \lambda kh}{\lambda kh} 
        = x'(0) + x_0 \frac{E_\lambda(kh)}{kh},
    \end{align*}
    which completes the proof.  
\end{proof}

\textbf{Proof of Theorem~\ref{thm:weak-form}.}

For notational convenience, let $h = \Delta t$.
    Observe that 
    \begin{align*}
        \theta^*_\text{weak} = \frac{- \sum_{k=1}^{m} \psi_k' \left[2kh \left( \frac{x_k- x_{-k}}{2kh} \right) + (\varepsilon_k - \varepsilon_{-k})\right]}{\psi_0(x_0 + \varepsilon_0) + \sum_{k=1}^{m} \psi_k \left[kh \left( \frac{x_k -x_0}{kh} - \frac{x_{-k}-x_0}{-kh} \right) + 2x_0 + (\varepsilon_k + \varepsilon_{-k})\right]}.
    \end{align*}
    It follows from Lemma~\ref{lemma:Taylor} that 
    the exact truncation errors are given by 
    \begin{align*}
        \frac{x_k - x_{-k}}{2kh} = x'_0 + x_0\frac{e_\lambda(kh)}{kh}, \quad 
        \frac{x_k -x_0}{kh} - \frac{x_{-k}-x_0}{-kh} = \frac{2x_0}{\lambda}\frac{e'_\lambda(kh)}{kh}
    \end{align*}
    where $e_\lambda(t) = \sinh(\lambda t) - \lambda t$.
    It then can be checked that we have
    \begin{align*}
        \theta^*_\text{weak} = 
        -\frac{x'_0 \sum_{k=1}^{m} 2kh \psi_k' + 2x_0\sum_{k=1}^m \psi_k' e_\lambda(kh)+ \sum_{k=-m}^m \varepsilon_k\psi_k'}{x_0\sum_{k=-m}^m\psi_k + \frac{2x_0}{\lambda}\sum_{k=1}^{m} \psi_k e'_\lambda(kh) + \sum_{k=-m}^m \varepsilon_k\psi_k}.
    \end{align*}
    
    Suppose that $1 = \sum_{k=-m}^m \psi_k = -\sum_{k=1}^m 2kh\psi_k'$.
    Then, $\theta^*_\text{weak} = \lambda - R$ where 
    \begin{align*}
        R = \frac{2x_0\sum_{k=1}^m \psi_k'e_\lambda(kh) + 2x_0\sum_{k=1}^{m} \psi_k e'_\lambda(kh) +\sum_{k=-m}^m\varepsilon_k(\lambda \psi_k + \psi_k')
        }{x_0  + \frac{2x_0}{\lambda}\sum_{k=1}^{m} \psi_k e'_\lambda(kh)+ \sum_{k=-m}^m \varepsilon_k\psi_k}.
    \end{align*}
    By letting $\bm{E} = (\sigma^{-1}\varepsilon_k)_{k=-m}^m$
    and $\Psi_{a,b} = (a\psi_k + b\psi_k')_{k=-m}^m$, we have
    \begin{align*}
        R = \frac{\frac{2x_0}{\sigma} \sum_{k=1}^m (\psi_k' e_\lambda(kh) + \psi_k e'_\lambda(kh)) 
        + \lambda \langle \bm{E}, \Psi_{1,\lambda^{-1}}\rangle
        }{\frac{x_0}{\sigma}(1 + \frac{2}{\lambda}\sum_{k=1}^{m} \psi_ke'_\lambda(kh))+ \langle \bm{E}, \Psi_{1, 0}\rangle},
    \end{align*}
    which completes the proof.  \hfill $\square$

\subsection{ Proof of Theorem~\ref{prop:weak-form-L1}}\label{app:proof_prop45}
It follows from $-1 = \sum_{k=1}^m 2kh\psi_k'$ that $\phi_1'=-\frac{1}{Sw_1}$.
Also, since $\psi_1' e_\lambda(\Delta t) + \psi_1 e'_\lambda(\Delta t) = 0$, we have
$\phi_1 = \frac{e_\lambda(\Delta t)}{w_1Se'_\lambda(\Delta t)}$.
Lastly, since $\sum_{k=-m}^m \psi_k = 1$, we have 
$\phi_0 = \frac{1}{w_0}\left[ 1 - 2\frac{e_\lambda(\Delta t)}{S e'_\lambda(\Delta t)} \right]$.
Since the first and second conditions in Theorem~\ref{thm:weak-form} are met,
it can be checked that 
\begin{align*}
    \theta^*_\text{weak} = \lambda - \lambda \frac{\langle \bm{E},\Psi_{1,\lambda^{-1}}\rangle}{\frac{x_0}{\sigma}(1 +  \frac{1}{\Delta t \lambda}e_\lambda(\Delta t)) + \langle \bm{E},\Psi_{1,0}\rangle}.
\end{align*}
Since $\lim_{\Delta t \to \infty} \frac{e_\lambda(\Delta t)}{\Delta t \lambda} = \infty$,
we have 
$\lim_{\Delta t \to \infty} \theta^*_{\text{weak}}(\Delta t) \overset{\text{a.s.}}{=} \lambda.$

Let $\bm{X}(\Delta t)=\langle \bm{E},\Psi_{1,\lambda^{-1}}\rangle$ 
and $\bm{Y}(\Delta t)=\langle \bm{E},\Psi_{1,0}\rangle$.
Since $\bm{E}$ is Gaussian, so are $\bm{X}$ and $\bm{Y}$.
It can be checked that $\text{Var}(\bm{X}(\Delta t)) = V(\lambda \Delta t) + \frac{1}{2(\lambda \Delta t)^2}$ and $\text{Var}(\bm{Y}(\Delta t)) = V(\lambda \Delta t)$
where 
\begin{align*}
    V(z) = \left(1 - \frac{1}{z}\frac{\sinh(z)-z}{\cosh(z)-1}\right)^2 + \frac{1}{2}\left(\frac{1}{z}\frac{\sinh(z)-z}{\cosh(z)-1}\right)^2.
\end{align*}
Note that $\lim_{z \to \pm \infty} V(z) = 1$, $\lim_{z \to 0} V(z) = 0.5$
and $0.5 \le V(z) \le 1$ for all $z$.
Thus, $\lim_{\Delta t \to \infty} \text{Var}(\bm{X}(\Delta t)) = \lim_{\Delta t \to \infty} \text{Var}(\bm{Y}(\Delta t)) = 1$,
$\lim_{\Delta t \to 0} \text{Var}(\bm{X}(\Delta t)) = \infty$,
and 
$\lim_{\Delta t \to 0} \text{Var}(\bm{Y}(\Delta t)) = 0.5$.

Observe that there exists a Guassian random variable $\bm{U}$, independent of $\bm{Y}$ such that $\bm{X} = \beta \bm{Y} + \bm{U}$
where $\beta = \frac{\text{Cov}(\bm{X},\bm{Y})}{\text{Var}(\bm{Y})}$.
Thus, we have
\begin{align*}
    \bm{W}:=\frac{\theta^*_\text{weak} - \lambda}{-\lambda}
    =  \frac{\bm{X}}{f(\Delta t) + \bm{Y}}
    =\beta \frac{\bm{Y}}{f(\Delta t) + \bm{Y}}
    + \frac{\bm{U}}{f(\Delta t)+ \bm{Y}}.
\end{align*}
where $f(\Delta t)= \frac{x_0}{\sigma}(1 +  \frac{1}{\Delta t \lambda}e_\lambda(\Delta t))$, which converges to 0 as $\Delta t \to \frac{x_0}{\sigma}$.
Since $\text{Var}(\bm{X}) \to \infty$ and $\text{Var}(\bm{X}) = \beta^2 \text{Var}(\bm{Y}) + \text{Var}(\bm{U})$, it can be checked that both $\beta^2 \to \infty$ and $\text{Var}(\bm{U}) \to \infty$.

Since $\bm{W}|\bm{Y}\sim N\left(\frac{\beta \bm{Y}}{f(\Delta t)+\bm{Y}}, \frac{\text{Var}(\bm{U})}{(f(\Delta t)+\bm{Y})^2}\right)$,
for any $M$, observe that 
\begin{align*}
    \text{Pr}\left( |\bm{W}| \le M | \bm{Y}\right) \le \frac{2M}{\sqrt{2\pi}}\frac{|f(\Delta t)+\bm{Y}|}{\sqrt{\text{Var}(\bm{U})}} \overset{\text{p}}{\to} 0.
\end{align*}
Since $\text{Pr}\left( |\bm{W}| \le M\right) = \mathbb{E}_{\bm{Y}}\left[\text{Pr}\left( |\bm{W}| \le M | \bm{Y}\right)\right]$, it follows from the dominated convergence theorem that 
$\lim_{\Delta t \to 0} \text{Pr}\left( |\bm{W}| \le M\right) = 0$, which completes the proof.

\hfill $\square$

    
\section{Implementation details}
\label{appendix:training_details}
In this appendix, we describe the implementation details of WGFINNs and GFINNs for all numerical examples in Section \ref{sec:results}, including the generation of training data, hyperparameter choices, and the neural network architectures.

\subsection{Two gas containers exchanging heat and volume}
\label{sec:app_GC}
This section presents the implementation details of WGFINNs and GFINNs used in Section \ref{sec:gas_containers}. The implementation of GFINNs follows the descriptions in \cite{zhang2022gfinns} and the corresponding publicly available source\footnote{\label{ft:gfinns}GitHub page: \url{https://github.com/zzhang222/gfinn_gc}}. 
The Runge-Kutta second/third order integrator (RK23) is used for time-step integration of the system dynamics during the training of GFINNs and the prediction phase of both methods.
Training took approximately \(4{,}000\) seconds of wall-clock time, corresponding to about \(170\text{K}\) iterations for WGFINNs and \(120\text{K}\) iterations for GFINNs. 

{\bf Data generation}: 
We use total $100$ trajectories from the time interval $[0,4]$ with time step $\Delta t = 0.01$. The initial conditions are randomly uniformly sampled from $[0.2,1.8]\times[-1,1]\times[1,3]\times[1,3]$.
For $100$ different initial conditions, the training and test data are generated using RK23 with the implementation provided in the GFINNs repository\textsuperscript{\ref{ft:gfinns}}. From these trajectories, $80$ are randomly selected for training, while the remaining $20$ trajectories are used for testing.

{\bf Hyperparameters and NN architectures}: 
For all noise levels, hyperparameters for each method are tuned to optimize prediction accuracy. 
Table \ref{tab:Hyperparameters_GC} presents the set of hyperparameters for both WGFINNs and GFINNs across all noise levels. For WGFINNs, test function parameters ($\ell$: support width, $p$: polynomial degree, $s$: overlap) are specified for the weak-form formulation. For GFINNs, batch size controls the number of trajectories used per training iteration. 
The initial learning rates (LR) for all models are selected from $\{10^{-6}, 10^{-5}, 10^{-4} , 10^{-3}\}$, and proper decay rates, decay steps are selected accordingly. 
The RBA parameters $(\eta^*, \gamma)$ are selected from $\{ (0 , 1), (0.001 , 0.999), (0.005 , 0.995), (0.01 , 0.99), (0.05 , 0.95), (0.1, 0.9)\}$. The batch size for GFINNs is selected from $\{500, 1000, 2000, 3000, 4000, 5000, \text{Full batch}\}$. The hyperparameters $\ell$, $p$, $s$ for WGFINNs lead to total $93$ and $35$ test functions for noise-free, noisy data cases respectively. 
The AdamW weight decay parameter is set to $10^{-2}$ for GFINNs trained on noisy data, and $0$ otherwise.
The weight matrix in the weighted loss is set as $W = \text{diag}(\frac{1}{\sigma_1}, \frac{1}{\sigma_2}, \ldots, \frac{1}{\sigma_d})$ for both WGFINNs, GFINNs, where $\sigma_j$ is defined in Section~\ref{sec:sw_rba}. 
For both methods, the four neural networks $E_{\text{NN}}$, $S_{\text{NN}}$, $L_{\text{NN}}$, and $M_{\text{NN}}$ consist of $4$ layers with $20$ neurons per hidden layer and hyperbolic tangent activation functions. Both models consist of \(3{,}646\) trainable parameters.

\begin{table}[!t]
\centering
\begin{tabular}{|c||c|c|c|c|c|c|}
\hline
\multicolumn{7}{|c|}{\textbf{WGFINNs}} \\
\hline
Noise  & LR & Decay rate & Decay steps & $\ell/p/s$ & $\eta^*$ (RBA) & $\gamma$ (RBA) \\
\hline
$0\%$ & 1e-3   & 0.99 & 1000 & 30/6/0.9  & 0.005 & 0.995 \\
\hline
$5\%$ & 1e-3  & 0.25  & 15000 & 80/4/0.9 & 0.01  & 0.99 \\
\hline
$10\%$ & 1e-3  & 0.1 & 15000 & 80/4/0.9 & 0.005 & 0.995 \\
\hline\hline
\multicolumn{7}{|c|}{\textbf{GFINNs}} \\
\hline
Noise& LR & Decay rate & Decay steps & Batch size & $\eta^*$ (RBA) & $\gamma$ (RBA) \\
\hline
$0\%$ & 1e-3 & 1 & N/A & 1000 & 0.001  & 0.999  \\
\hline
$5\%$ & 1e-3 & 0.1 & 5000 & 1000 & 0.1 & 0.9 \\
\hline
$10\%$ &  1e-3 & 0.75 & 10000 & 1000 & 0.001  & 0.999  \\
\hline
\end{tabular}
\caption{Example \ref{sec:gas_containers}: Hyperparameters for WGFINNs and GFINNs at varying noise levels. The decay rate $\alpha$ means the multiplication by $\alpha$ is applied to the current learning rate.}
\label{tab:Hyperparameters_GC}
\end{table}

{\bf Calibration of learned energy ($E_{\text{NN}}$) and entropy ($S_{\text{NN}}$)}:
The GENERIC formalism allows infinitely many module combinations $(\tilde{E},\tilde{S},\tilde{L},\tilde{M})$ to produce identical dynamics for a given $(E,S,L,M)$. For instance, any affine transformation $\tilde{S} = cS + d$ with $\tilde{M} = c^{-1}M$ preserves the governing equation. This means the learned $E_{\text{NN}}$ and $S_{\text{NN}}$ may differ from ground truth by unknown scaling and translation. To facilitate meaningful comparison with the ground truth, we therefore apply an affine 
calibration $\hat{E} = a E_{\text{NN}} + b$ and $\hat{S} = c S_{\text{NN}} + d$ to the inferred quantities, where the coefficients $(a, b)$, $(c, d)$ are determined by least squares regression against the ground truth over a reference domain. For the entropy trajectory plots, we additionally anchor the 
calibration at $t = 0$ so that all models share the same initial value, with the scale $c$ optimally fitted to the remaining trajectory.

\subsection{Thermoelastic double pendulum}
\label{sec:app_DP}
This section describes the implementation details of WGFINNs and GFINNs employed in Section~\ref{sec:double_p}.
The GFINNs implementation follows the methodology presented in \cite{zhang2022gfinns}. 
The RK23 is used to perform time integration of the system dynamics during the training of GFINNs and prediction phase of both methods.
The training process lasted approximately \(4{,}000\) seconds of wall-clock time, which corresponds to around \(130\text{K}\) iterations for WGFINNs and \(90\text{K}\) iterations for GFINNs. 

{\bf Data generation}: 
We consider $100$ trajectories over time interval $[0,10]$ with time step $\Delta t = 0.025$. Following \cite{zhang2022gfinns}, the initial conditions are uniformly sampled from $[0.9,1.1]\times[-0.1,0.1]\times[2.1,2.3]\times[-0.1,0.1]\times[-0.1,0.1]\times[1.9,2.1]\times[0.9,1.1]\times[-0.1,0.1]\times[0.9,1.1]\times[0.1,0.3]$.
We use $80$ trajectories for training data and the remainder for testing.
The training and test data are generated using the RK23 integrator from the $100$ distinct initial conditions.

{\bf Hyperparameters and neural network architectures}:
For each noise level, the hyperparameters of WGFINNs and GFINNs are independently tuned to achieve the best predictive accuracy.
The selected hyperparameter configurations for all noise settings are summarized in Table~\ref{tab:Hyperparameters_DP}.
In WGFINNs, the weak-form formulation introduces additional test-function-related parameters, namely the support width $\ell$, polynomial degree $p$, and overlap ratio $s$.
For GFINNs, the batch size determines the number of trajectories used in each training iteration.
The initial learning rate is chosen from $\{10^{-6}, 10^{-5}, 10^{-4}, 10^{-3}\}$, with decay rates and decay steps selected accordingly.
The RBA parameters $(\eta^*, \gamma)$ are selected from
\begin{equation*}
\{ (0 , 1), (0.001 , 0.999), (0.005 , 0.995), (0.01 , 0.99), (0.05 , 0.95), (0.1, 0.9)\}.
\end{equation*}
For GFINNs, the batch size is chosen from
${500, 1000, 2000, 3000, 4000, 5000, \text{Full batch}}$.
The selected WGFINNs parameters $(\ell, p, s)$ result in a total of $131$ and $52$ test functions for noise-free and noisy data cases respectively.
For GFINNs trained on noisy data, the AdamW weight decay parameter is set to $10^{-2}$, while no weight decay is applied in other cases.
For both WGFINNs and GFINNs, the weighted loss employs the diagonal matrix
$W = \mathrm{diag}\left(\frac{1}{\sigma_1}, \frac{1}{\sigma_2}, \ldots, \frac{1}{\sigma_d}\right)$,
where $\sigma_j$ is defined in Section~\ref{sec:sw_rba}.
Both models use the same network architecture: the four neural networks $E_{\mathrm{NN}}$, $S_{\mathrm{NN}}$, $L_{\mathrm{NN}}$, and $M_{\mathrm{NN}}$ each consist of $5$ layers with $30$ neurons per hidden layer yielding total \(20{,}742\) NN parameters and hyperbolic tangent activation functions.

\begin{table}[!t]
\centering
\begin{tabular}{|c||c|c|c|c|c|c|}
\hline
\multicolumn{7}{|c|}{\textbf{WGFINNs}} \\
\hline
Noise  & LR & Decay rate & Decay steps & $\ell/p/s$ & $\eta^*$ (RBA) & $\gamma$ (RBA) \\
\hline
$0\%$ & 1e-3   & 0.99 & 1000 & 10/4/0.7  & 0.0 & 1 \\
\hline
$5\%$ & 1e-3  & 0.1 & 7500 & 40/2/0.95 & 0.01 &  0.99\\
\hline
$10\%$ & 1e-3  & 0.1 & 7500 & 40/2/0.95  & 0.05 & 0.95 \\
\hline\hline
\multicolumn{7}{|c|}{\textbf{GFINNs}} \\
\hline
Noise& LR & Decay rate & Decay steps & Batch size & $\eta^*$ (RBA) & $\gamma$ (RBA) \\
\hline
$0\%$ & 1e-3 & .99 & 1000 & 3000 & 0.001  & 0.999  \\
\hline
$5\%$ & 1e-3 & 0.1 & 10000 & 1000 & 0.001&  0.999 \\
\hline
$10\%$ & 1e-3 & 0.1 & 5000 & 1000 & 0.01  & 0.99  \\
\hline
\end{tabular}
\caption{Example \ref{sec:double_p}: Hyperparameters for WGFINNs and GFINNs at varying noise levels. The decay rate $\alpha$ means the multiplication by $\alpha$ is applied to the current learning rate.}
\label{tab:Hyperparameters_DP}
\end{table}

\subsection{Linearly damped system}
\label{sec:app_LD}
This section details the implementation of WGFINNs and GFINNs considered in Section~\ref{sec:linearly_damped}.
The system dynamics are integrated in time using the RK23 scheme for the training of GFINNs and prediction of both methods.
Training was carried out for approximately \(4{,}000\) seconds in wall-clock time, corresponding to about \(150\text{K}\) iterations for WGFINNs and \(120\text{K}\) iterations for GFINNs.

{\bf Data generation}: 
We generate $100$ trajectories on the time interval $[0,4]$ using a time step of $\Delta t = 0.01$.  
The initial states are randomly drawn from the domain $[1.5,2.5]\times[-0.2,0.2]\times\{0\}$ using uniform distribution.  
Among these, $80$ trajectories are used for training, while the remaining ones are used for evaluation.
The datasets for training and testing are obtained by integrating the system from $100$ different initial conditions using the RK23 integrator 
referring to the source code in \footnote{\label{ft:ld_github}GitHub page: \url{https://github.com/sxconly/GENERIC-Integrator}} associated with \cite{shang2020structure}.

{\bf Hyperparameters and neural network architectures}:
For each noise level, the hyperparameters of WGFINNs and GFINNs are tuned separately to optimize predictive accuracy.
The resulting hyperparameter choices for all noise settings are reported in Table~\ref{tab:Hyperparameters_LD}.
In WGFINNs, the weak-form formulation introduces additional parameters associated with the test functions, including the support width $\ell$, the polynomial degree $p$, and the overlap ratio $s$.
In GFINNs, the batch size specifies the number of trajectories processed in each training iteration.
The initial learning rate is selected from $\{10^{-6}, 10^{-5}, 10^{-4}, 10^{-3}\}$, with appropriate decay rates and decay steps applied.
The RBA parameters $(\eta^*, \gamma)$ are chosen from
\begin{equation*}
\{ (0 , 1), (0.001 , 0.999), (0.005 , 0.995), (0.01 , 0.99), (0.05 , 0.95), (0.1, 0.9)\}.
\end{equation*}
For GFINNs, the batch size is selected from
${500, 1000, 2000, 3000, 4000, 5000, \text{Full batch}}$.
The selected WGFINNs parameter set $(\ell, p, s)$ yields a total of $25$ test functions.
When training GFINNs on noisy data, the AdamW optimizer is used with a weight decay coefficient of $10^{-2}$, whereas no weight decay is applied in the noise-free setting.
For both WGFINNs and GFINNs, the weighted loss function employs the diagonal matrix
$W = \mathrm{diag}\left(\frac{1}{\sigma_1}, \frac{1}{\sigma_2}, \ldots, \frac{1}{\sigma_d}\right)$,
where $\sigma_j$ is defined in Section~\ref{sec:sw_rba}.
Both models adopt an identical network architecture: each of the four neural networks $E_{\mathrm{NN}}$, $S_{\mathrm{NN}}$, $L_{\mathrm{NN}}$, and $M_{\mathrm{NN}}$ consists of $5$ layers with $50$ neurons per hidden layer that correspond to \(24{,}577\) NN parameters, and hyperbolic tangent activation functions.

\begin{table}[!t]
\centering
\begin{tabular}{|c||c|c|c|c|c|c|}
\hline
\multicolumn{7}{|c|}{\textbf{WGFINNs}} \\
\hline
Noise  & LR & Decay rate & Decay steps & $\ell/p/s$ & $\eta^*$ (RBA) & $\gamma$ (RBA) \\
\hline
$0\%$ & 1e-3   & 0.95 & 1000 & 60/4/0.8  & 0.0 & 1.0\\
\hline
$5\%$ & 1e-4  & 0.95 & 1000 &  60/4/0.8  & 0.0 &  1.0\\
\hline
$10\%$ & 1e-4  & 0.85 & 1000 &  60/4/0.8  & 0.0 & 1.0\\
\hline\hline
\multicolumn{7}{|c|}{\textbf{GFINNs}} \\
\hline
Noise& LR & Decay rate & Decay steps & Batch size & $\eta^*$ (RBA) & $\gamma$ (RBA) \\
\hline
$0\%$ & 1e-3 & .95 & 1000 & 3000 & 0.0  & 1.0 \\
\hline
$5\%$ & 1e-5 & 0.1 & 1000 & 3000 & 0.0 &  1.0 \\
\hline
$10\%$ & 1e-5 & 0.1 & 1000 & 3000 & 0.0  & 1.0 \\
\hline
\end{tabular}
\caption{Example \ref{sec:linearly_damped}: Hyperparameters for WGFINNs and GFINNs at varying noise levels. The decay rate $\alpha$ means the multiplication by $\alpha$ is applied to the current learning rate.}
\label{tab:Hyperparameters_LD}
\end{table}

\subsection{1D-1V Vlasov Equation: electric field dynamics}
\label{sec:app_plasma}
This section presents the implementation details of WGFINNs and GFINNs for Section~\ref{sec:plasma}. The system dynamics are integrated in time using the RK23 scheme during training of GFINNs and the prediction stage of both models.
The models were trained for approximately $4{,}000$ seconds of wall-clock time, which corresponds to approximately $240\text{K}$ and $160\text{K}$ iterations for WGFINNs and GFINNs, respectively.

{\bf Data generation}: 
The dataset consists of $440$ trajectories from uniformly spaced parameter values $\boldsymbol{\mu}$ over the time interval $[8,10]$ with 
time step $\Delta t = 0.02$. Among these, $429$ trajectories are used for training, and the remaining $11$ are held out for testing.
The training and test data are generated using the HyPar solver \footnote{\label{ft:hypar}GitHub page: \url{https://hypar.github.io/}} with a WENO spatial
discretization and the fourth-order Runge-Kutta explicit time integration ($\Delta t = 0.005$, $t \in [0,10]$). The electric field is then computed by solving the Poisson equation in a periodic domain using Fourier transforms. From this original dataset, the training and test data are subsampled so that $\Delta t = 0.02$, $T = [8, 10]$.

{\bf Hyperparameters and neural network architectures}:
The hyperparameters of WGFINNs and GFINNs are independently tuned for each 
noise level to achieve optimal predictive performance, and the selected 
configurations are summarized in Table~\ref{tab:Hyperparameters_PL}.
The initial learning rate is chosen from $\{10^{-6}, 10^{-5}, 10^{-4}, 10^{-3}\}$, 
with decay rates and decay steps selected accordingly.
The RBA parameters $(\eta^*, \gamma)$ are selected from
\begin{equation*}
\{ (0 , 1), (0.001 , 0.999), (0.005 , 0.995), (0.01 , 0.99), (0.05 , 0.95), (0.1, 0.9)\}.
\end{equation*}
For GFINNs, the batch size is chosen from 
$\{500, 1000, 2000, 3000, 4000, 5000, \text{Full batch}\}$.
The hyperparameters $\ell$, $p$, $s$ for WGFINNs yields total $45$ and $20$ test functions for noise-free and noisy data cases respectively.
A weight decay coefficient of $10^{-2}$ is applied to the AdamW optimizer when 
training WGFINNs on $5\%$ noisy data, while no weight decay is used in the noise-free case.
The weighted loss function for both methods uses the diagonal matrix
$W = \mathrm{diag}\left(\frac{1}{\sigma_1}, \frac{1}{\sigma_2}, \ldots, \frac{1}{\sigma_d}\right)$,
where $\sigma_j$ is defined in Section~\ref{sec:sw_rba}.
Both WGFINNs and GFINNs share an identical network architecture, where each of 
the four neural networks $E_{\mathrm{NN}}$, $S_{\mathrm{NN}}$, $L_{\mathrm{NN}}$, 
and $M_{\mathrm{NN}}$ comprises $5$ layers with $80$ neurons per hidden 
layer and hyperbolic tangent activation functions. Both methods consist of $92{,}194$ trainable parameters.


\begin{table}[t]
\centering
\begin{tabular}{|c||c|c|c|c|c|c|}
\hline
\multicolumn{7}{|c|}{\textbf{WGFINNs}} \\
\hline
Noise  & LR & Decay rate & Decay steps & $\ell/p/s$ & $\eta^*$ (RBA) & $\gamma$ (RBA) \\
\hline
$0\%$ & 1e-3   & 0.98 & 500 & 10/3/0.9  & 0.0 & 1.0\\
\hline
$3\%$ & 1e-3  & 0.95 & 1000 &  20/3/0.9  & 0.005 &  0.995\\
\hline
$5\%$ & 1e-3  & 0.95 & 1000 &  20/3/0.9  & 0.05 & 0.95\\
\hline\hline
\multicolumn{7}{|c|}{\textbf{GFINNs}} \\
\hline
Noise& LR & Decay rate & Decay steps & Batch size & $\eta^*$ (RBA) & $\gamma$ (RBA) \\
\hline
$0\%$ & 1e-3 & 0.99 & 1000 & 5000 & 0.0  & 1.0 \\
\hline
$3\%$ & 1e-3 & 0.99 & 1000 & 5000 & 0.001 &  0.999 \\
\hline
$5\%$ & 1e-4 & 0.95 & 1000 & 2000 & 0.001  & 0.999 \\
\hline
\end{tabular}
\caption{Example \ref{sec:plasma}: Hyperparameters for WGFINNs and GFINNs at varying noise levels. The decay rate $\alpha$ means the multiplication by $\alpha$ is applied to the current learning rate.}
\label{tab:Hyperparameters_PL}
\end{table}
\end{appendices}


\clearpage

\bibliographystyle{siamplain}

\bibliography{referencesWGFINNs}



\end{document}